\documentclass{article}

\usepackage{microtype}
\usepackage{graphicx}
\usepackage{subfig}
\usepackage{booktabs} 

\usepackage{hyperref}

\usepackage{chngcntr}
\usepackage{apptools}
\AtAppendix{\counterwithin{lemma}{section}}

\usepackage{url}  
\usepackage{graphicx}  
\usepackage{amsmath,amsfonts}
\usepackage{bm}
\usepackage{mathtools}
\usepackage{amssymb}
\usepackage{enumitem}

\DeclareMathOperator*{\argmax}{\arg\max}

\newtheorem{lemma}{Lemma}

\newtheorem{remark}{Remark}
\newtheorem{theorem}{Theorem}
\newcommand{\xl}[1]{\textcolor{black}{#1}}
\newcommand{\pch}[1]{\textcolor{black}{#1}}

\usepackage[accepted]{icml2019}

\icmltitlerunning{Stay With Me: Lifetime Maximization Through Heteroscedastic Linear Bandits With Reneging}

\makeatletter
\DeclareRobustCommand{\qed}{%
  \ifmmode 
  \else \leavevmode\unskip\penalty9999 \hbox{}\nobreak\hfill
  \fi
  \quad\hbox{\qedsymbol}}
\newcommand{\openbox}{\leavevmode
  \hbox to.77778em{%
  \hfil\vrule
  \vbox to.675em{\hrule width.6em\vfil\hrule}%
  \vrule\hfil}}
\newcommand{\qedsymbol}{\openbox}
\makeatletter
\newcommand{\printfnsymbol}[1]{%
  \textsuperscript{\@fnsymbol{#1}}%
}

\newenvironment{proof}[1][\proofname]{\par
  \normalfont
  \topsep6\p@\@plus6\p@ \trivlist
  \item[\hskip\labelsep\itshape
    #1.]\ignorespaces
}{%
  \qed\endtrivlist
}
\newcommand*\xbar[1]{%
   \hbox{%
     \vbox{%
       \hrule height 0.5pt 
       \kern0.5ex
       \hbox{%
         \kern-0.1em
         \ensuremath{#1}%
         \kern-0.1em
       }%
     }%
   }%
}
\newcount\colveccount
\newcommand*\colvec[1]{
        \global\colveccount#1
        \begin{pmatrix}
        \colvecnext
}
\def\colvecnext#1{
        #1
        \global\advance\colveccount-1
        \ifnum\colveccount>0
                \\
                \expandafter\colvecnext
        \else
                \end{pmatrix}
        \fi
}
\newcommand{\proofname}{Proof}
\newcommand\norm[1]{\left\lVert#1\right\rVert}
\DeclarePairedDelimiter\abs{\lvert}{\rvert}%
\DeclarePairedDelimiterX{\bkt}[1]{(}{)}{ #1}
\DeclarePairedDelimiterX{\sbkt}[1]{[}{]}{ #1}
\newcommand\inv[1]{#1\raisebox{1.15ex}{$\scriptscriptstyle-\!1$}}

\begin{document}

\twocolumn[
\icmltitle{Stay With Me: Lifetime Maximization Through \\ Heteroscedastic Linear Bandits With Reneging}



\icmlsetsymbol{equal}{*}

\begin{icmlauthorlist}
\icmlauthor{Ping-Chun Hsieh}{equal,to}
\icmlauthor{Xi Liu}{equal,to}
\icmlauthor{Anirban Bhattacharya}{goo}
\icmlauthor{P. R. Kumar}{to}
\end{icmlauthorlist}

\icmlaffiliation{to}{Department of Electrical and Computer Engineering, Texas A\&M University, College Station, USA}
\icmlaffiliation{goo}{Department of Statistics, Texas A\&M University, College Station, USA}

\icmlcorrespondingauthor{Ping-Chun Hsieh}{pingchun.hsieh@ tamu.edu}
\icmlcorrespondingauthor{Xi Liu}{xiliu.tamu@gmail.com}

\icmlkeywords{Machine Learning, ICML}

\vskip 0.3in
]



\printAffiliationsAndNotice{\icmlEqualContribution} 

\begin{abstract}
Sequential decision making for lifetime maximization is a critical problem in many real-world applications, such as medical treatment and portfolio selection.
In these applications, a ``reneging'' phenomenon, where participants may disengage from future interactions after observing an unsatisfiable outcome, is rather prevalent.
To address the above issue, this paper proposes a model of heteroscedastic linear bandits with reneging, which allows each participant to have a distinct ``satisfaction level," with any interaction outcome falling short of that level resulting in that participant reneging. 
Moreover, it allows the variance of the outcome to be context-dependent. 
Based on this model, we develop a UCB-type policy, namely HR-UCB, and prove that it achieves $\mathcal{O}\big(\sqrt{{T}(\log({T}))^{3}}\big)$ regret. Finally, we validate the performance of HR-UCB via simulations. 
\end{abstract}
\vspace{-5mm}

\section{Introduction}\label{section:intro}

Sequential decision problems commonly arise in a large number of real-world applications. To name a few, in treatment to extend the life of people with terminal illnesses, doctors are required to make decisions on which treatments are used for patients periodically. In portfolio selection, fund managers need to decide which portfolios are recommended to their customers every time. In cloud computing services, the cloud platform has to determine the resources allocated to customers given specific requirements of their programs. Multi-armed Bandits (MAB) \cite{auer2002finite} and one of its most famous variants ``contextual bandits'' \cite{abbasi2011improved} have been extensively used to model such problems. In the modeling, available choices are referred to as ``arms'' and a decision is regarded as a ``pull'' of the corresponding arm. The decision is evaluated through rewards that depend on the goal of the interaction.

In the aforementioned applications and services, a phenomenon that participants may disengage from future interactions commonly exist. Such behavior is referred to as ``churn'', ``unsubscribe'' or ``reneging'' in literature \cite{liu2018semi}. For instance, patients fail to survive the illnesses or are unable to take more treatments due to the deterioration of physical condition \cite{mchugh2015extending}. In portfolio selection, fund managers earn money from customer enrollment of the selection service. The return of the selection may turn out to be loss and thus the customer loses trust to the manager and stops using the service \cite{huo2017risk}. Similarly, in the cloud computing services, the customer may feel the resource is not well allocated and leads to an unsatisfied throughput, thus switching to another service provider \cite{Ding2013}. In other words, the participant \footnote{For simplicity, in this paper, we use the terms \emph{participant}, \emph{user}, \emph{customer}, and \emph{patients} interchangeably.} of the interaction has a ``lifetime'' that can be defined as the total number interactions between the participant and a service provider until reneging. The larger the number is, the ``longer'' participant stay with the provider. Customer lifetime has been recognized as a critical metric to evaluate the success of many applications such as all above application as well as the e-commerce \cite{Theocharous2015}. Moreover, as well known, the acquisition cost for a new customer is much higher than an existing customer \cite{liu2018semi}. Therefore, within the applications and services, one particular vital goal is to maximize the lifetime of customers. Our focus in this paper is to learn an optimal decision policy that maximizes the lifetime of participants in interactions.

We consider reneging behavior based on two observations. First, in all above scenarios, the decision maker is usually able to observe the outcome of following their suggestion, e.g., the physical condition of the patients after the treatment, the money earned from purchasing the suggested portfolio in the account, and the throughput rate of running the programs. Second, we observe that the participants in those applications are willing to reveal their satisfaction level to the outcome of the suggestion.  For instance, patients will let doctors know their expectations to the treatment in physician visits. Customers are willing to inform fund managers how much money they can afford to lose. Cloud users share with the service providers their requirements of throughput performance. We consider that the outcome of following the suggestion is a random variable drawn from an unknown distribution that may vary under different contexts. If the outcome falls below the satisfaction level, the customer quits all future interactions; otherwise, the customer stays. That being said, the reneging risk is the chance that the outcome drawn from an unknown distribution falls below some threshold. Thus, learning the unknown outcome distribution plays a critical role in optimal decision making.

Learning the outcome distribution of following the suggestion can be highly challenging due to ``heteroscedasticity'', which means the variability of the outcome varies across the range of predictors. Many previous studies of the aforementioned applications have pointed out that the distribution of the outcome can be heteroscedastic. In treatment to a patient, it has been found the physical condition after treatment can be highly heteroscedastic \cite{towse2015understanding,buzaianu2018two}. Similarly, in portfolio selection \cite{omari2018currency,ledoit2003improved,jin2018large}, it is even more common that the return of investing a selected portfolio is heteroscedastic. In cloud service, it has been repeatedly observed that the throughput and responses of the server can be highly heteroscedastic \cite{somu2018improved,niu2011understanding,Cheng1999}. In bandits setting, it means that both the mean value and the variance of the outcome depend on ``context'' which represents the decision and the customer. Since the reneging risk is the chance that the outcome is below the satisfaction level, accurately estimating it now requires estimation of both mean and variance. Such property makes it more difficult to learn the distribution.

While MAB and contextual bandits have been successfully applied to many sequential decision problems, they are not directly applicable to the lifetime maximization problem due to two major limitations. First, most of them neglect the phenomenon of reneging that is common in real-world applications. As a result, their objective is to maximize the accumulated rewards collected from endless interactions. As a comparison, our goal is to maximize the total number of interactions where each time of interaction faces some reneging risk. Due to that reason, conventional approaches such as LinUCB \cite{chu2011contextual} will have poor performance in solving the lifetime maximization problem (see Section \ref{section:simulation} for a comparison). Second, previous studies have usually assumed that the underlying distribution involved in the problem is homoscedastic, i.e., its variance is independent of contexts. Unfortunately, this assumption can be easily invalid due to the presence of heteroscedasticity in the motivated examples considered above, e.g., patients' health condition, portfolio return, and throughput rate. 

The line of MAB research that is most relevant to the problem is bandits models with risk management, e.g., variance minimization \cite{sani2012risk} and value-at-risk maximization \cite{szorenyi2015qualitative,asaf2018general,chaudhuriquantile}. However, the risks those studies handle models the large fluctuation of collected rewards and have no impact on the lifetimes of bandits. This makes them unable to be applied to our problem. Another category of relevant research is conservative bandits \cite{kazerouni2017conservative,wu2016conservative}, in which a choice will only be considered if it guarantees that the \emph{overall} performances outperforms $1-\alpha$ of baselines'. Unfortunately, our problem has a higher degree of granularity, i.e., to avoid reneging, \emph{individual} performance (performance of each choice) is above some satisfaction level. Moreover, none of them considers data heteroscedasticity. (A more careful review and comparison are given in Section \ref{section:related})

To overcome all these limitations, we propose a novel model of contextual bandits that addresses the challenges arising from reneging risk and heteroscedasticity in the lifetime maximization problem. We call the model ``heteroscedastic linear bandits with reneging''. 

\noindent \textbf{Contributions.} Our research contributions are as follows:
\vspace{-2mm}

\begin{enumerate}[leftmargin=1pt,itemindent=12pt,itemsep=2pt,parsep=0pt,topsep=0pt,partopsep=0pt]
\item Lifetime maximization is an important problem in many real-world applications but not taken into account in the existing bandit models. We investigate the two characters of the problem in aforementioned applications: reneging risk and willingness to reveal satisfaction level and propose a behavior model of reneging. 

\item In view of the two characters, we formulate a novel bandits model for lifetime maximization under heteroscedasticity. It is based on our model of reneging behavior and is dubbed ``heteroscedastic linear bandits with reneging.''

\item To solve the proposed model, we develop a UCB-type policy, called HR-UCB, that is proved to achieve a \pch{$\mathcal{O}\big(\sqrt{{T}(\log({T}))^{3}}\big)$} regret bound with high probability. We evaluate the HR-UCB via comprehensive simulations. The simulation results demonstrate that our model satisfies our expectation of regret and outperforms conventional UCB that ignores reneging and more complex model such as Episodic Reinforcement Learning (ERL).
\end{enumerate}

\vspace{-2mm}

\section{Related Work} \label{section:related}

There are mainly two lines of research related to our work: bandits with risk management and conservative bandits.

\noindent \textbf{Bandits with Risk Management.} Reneging can be viewed as a type of risk that the decision maker tries to avoid. 
The risk management in bandit problems has been studied in terms of variance and quantiles. 
In \cite{sani2012risk}, mean-variance models to handle risk are studied, where the risk refers to the variability of collected rewards. The difference from conventional bandits is that the objective to be maximized is a linear combination of mean reward and variance. Subsequent studies \cite{szorenyi2015qualitative,asaf2018general} propose a quantile (value at risk) to replace the mean-variance objective. While these studies investigate optimal policies under risk, the risks they handle are different from ours, in the sense that the risks usually relate to variability of rewards and have no impact on the lifetime of bandits. Moreover, their approaches to handle the risk are based on more straightforward statistics, while, in our problem, the reneging risk is relatively complex, i.e., it comes from the probability that the outcome of following a suggestion is below a satisfaction level. Therefore, their models cannot be used to solve our problem.

\noindent \textbf{Conservative Bandits.}
In contrast to those works, \emph{conservative bandits} \cite{kazerouni2017conservative,wu2016conservative} control the risk by requiring that the accumulated rewards while learning the optimal policy be above those of baselines. Similarly, in \cite{sun2017safety}, each arm is associated with some risk; safety is guaranteed by requiring the accumulated risk to be below a given budget. 
Unfortunately, our problem has a higher degree of granularity. The participants in our problem are more sensitive to bad suggestions. One time of bad decision may incur reneging and brings the interactions to an end, e.g., one bad treatment makes a patient die. Moreover, their models assume homoscedasticity, while we allow the variance to depend on the context.

The satisfaction level in our model has the flavor of thresholding bandits. Different from us, the thresholds in the existing literature are mostly used to model reward generation. For instance, in \cite{abernethy2016threshold}, an action induces a unit payoff if the sampled outcome exceeds a threshold. In \cite{jain2018firing}, no rewards can be collected until the total number of successes exceeds the threshold. 

In terms of the problem in this paper, the most relevant one that has been studied is in \cite{schmit2018learning}. Compared to it, our paper has three salient differences. First, it has a very different setting of reneging modeling: each decision is represented by a real number; reneging happens as long as the pulled arm exceeds a threshold. As a comparison, we represent each decision by a high-dimensional context vector; reneging happens if the outcome of following a suggestion is not satisfying. Second, it couples the reneging with the reward generation. The ``rewards'' in our modeling can be regarded as the lifetime while the reneging is separately captured by the outcome distribution. Third, it fails to take into account the data heteroscedasticity in the aforementioned applications. 
By contrast, we investigate the impacts of that and our model well addresses it.

In terms of bandits under heteroscedasticity, to the best of our knowledge, only one very recent paper discusses that \cite{kirschner2018information}. 
Compared to it, our paper has two salient differences. 
First, we address heteroscedasticity under the presence of reneging. The presence of reneging makes the learning problem more challenging as the learner has to always be prepared that plans for the future may not be carried out. Second, the solution in \cite{kirschner2018information} is based on information directed sampling. In contrast, in this paper, we present a heteroscedastic UCB policy that is efficient, easier to implement, and can achieve sub-linear regret. \xl{The reneging problem can also be approximated by an infinite-horizon ERL problem~\cite{modi2018markov,hallak2015contextual}. Compared to it, our solution has two distinct features: (a) the reneging behavior and heteroscedasticity are explicitly addressed in our model, (b) the context information is leveraged in learning policy design.}

\section{Problem Formulation} \label{section:problem}
In this section, we describe the formulation of the heteroscedastic linear bandits with reneging.
To incorporate reneging behavior into the bandit model, we address the problem in the following stylized manner:
The users arrive at the decision maker one after another and are indexed by $t=1,2,\cdots$.
For each user $t$, the decision maker interacts with the user in discrete \emph{rounds} by selecting one action in each round sequentially until the user $t$ reneges on interacting with the decision maker. 
Let $s_t$ denote the total number of rounds experienced by the user $t$.
Note that $s_t$ is a stopping time which depends on the reneging mechanism that will be described shortly. 
Since the decision maker interacts with one user at a time, all the actions and the corresponding outcomes regarding user $t$ are determined and observed, before the next user $t+1$ arrives.

Let $\mathcal{A}$ be the set of available actions of the decision maker.
Upon the arrival of each user $t$, the decision maker observes a set of \emph{contexts} $\mathcal{X}_{t}=\{x_{t,a}\}_{a\in \mathcal{A}}$, where each context $x_{t,a}\in\mathcal{X}_{t}$ summarizes the pair-wise relationship\footnote{For example, in recommender systems, one way to construct a such pair-wise context is to concatenate the feature vectors of each individual user and each individual action.} between the user $t$ and the action $a$.
Without loss of generality, we assume that for any user $t$ and any action $a$, we have $\norm{x_{t,a}}_2\leq1$, where $\|\cdot\|_2$ denotes the $\ell_{2}$-norm.
After observing the contexts, the decision maker selects an action $a\in\mathcal{A}$ and observes a random outcome $r_{t,a}$.
We assume that the outcomes $r_{t,a}$ are conditionally independent random variables given the contexts and are drawn from an outcome distribution that satisfies:
\begin{align}
    &r_{t,a} := \theta_{*}^{\top}x_{t,a}  + \varepsilon(x_{t,a})\label{equation: reward distribution 1}\\
    &\varepsilon(x_{t,a}) \sim \mathcal{N}\big(0,\sigma^{2}(x_{t,a})\big) \label{equation: reward distribution 2}\\
    &\sigma^{2}(x_{t,a}) := f(\phi_{*}^{\top}x_{t,a}),\label{equation: reward distribution 3}
\end{align}
where $\mathcal{N}(0,\sigma^{2})$ denotes the Gaussian distribution with zero mean and variance $\sigma^{2}$, and $\theta_*,\phi_*\in\mathbb{R}^{d}$ are unknown, but known to have the norm bounds as $||\theta_*||_2 \leq 1$ and $||\phi_*||_{2}\leq L$. \xl{Although, for simplicity of discussion, here we focus on Gaussian noise, all of our analysis can be extended to sub-Gaussian outcome distribution of the form $\psi_\sigma(x) = (1/\sigma) \psi((x-\mu)/\sigma)$, where $\psi$ is a known sub-Gaussian density with unknown parameters $\mu, \sigma$. This family includes truncated distributions and mixtures, thus allowing multi-modality and skewness.} The parameter vectors $\theta_*\in \mathbb{R}^{d}$ and $\phi_*\in \mathbb{R}^{d}$ will be learned by the decision maker during interactions with the users. 
\pch{The function $f(\cdot):\mathbb{R}\to\mathbb{R}$ is assumed to be a known linear function with a finite positive slope $M_f$ such that $f(z)\geq 0$, for all $z\in[-L,L]$. 
One example that satisfies the above conditions is $f(z)=z+L$. }
Note that the mean and variance of the outcome distribution satisfy
\begin{align}
    \mathbb{E}[r_{t,a}|x_{t,a}] &:= \theta_{*}^{\top}x_{t,a},\\
    \mathbb{V}[r_{t,a}|x_{t,a}] &:= f(\phi_{*}^{\top}x_{t,a}).
\end{align}
Since $\phi_{*}^{\top}x_{t,a}$ is bounded over all possible $\phi_{*}$ and $x_{t,a}$, we know that $f(\phi_{*}^{\top}x_{t,a})$ is also bounded, i.e. $f(\phi_{*}^{\top}x_{t,a})\in[\sigma^{2}_{\min},\sigma^{2}_{\max}]$ for some $\sigma_{\min},\sigma_{\max}>0$, for all $\phi_{*}$ and $x_{t,a}$ defined above.
This also implies that $\varepsilon(x_{t,a})$ is $\sigma_{\max}^{2}$-sub-Gaussian, for all $x_{t,a}$.

The minimal expectation in an interaction of a user is characterized by its \emph{satisfaction level}. 
Let $\beta_{t} \in \mathbb{R}$ denote the satisfaction level of user $t$. 
We assume that satisfaction levels of users, like the pair-wise contexts, are available before interacting with them. 
Denote by $r^{(i)}_{t}$ the observed outcome at round $i$ of user $t$. 
When $r^{(i)}_{t}$ is below $\beta_{t}$, reneging occurs and the user drops out from any future interaction. 
Suppose that at round $i$, action $a$ is selected for user $t$, then the risk that reneging occurs is
\begin{align}\label{equation: probability}
    \mathbb{P}(r^{(i)}_{t} < \beta_{t}|x_{t,a}) = \Phi\Big(\dfrac{\beta_{t}- \theta_{*}^{\top}x_{t,a}}{\sqrt{f(\phi_{*}^{\top}x_{t,a})}}\Big),
\end{align}
where $\Phi(\cdot)$ is the cumulative density function (CDF) for $\mathcal{N}(0,1)$. Without loss of generality, we also assume that ${\beta_{t}}$ is lower bounded by $-B$ for some $B>0$. 
Recall that $s_{t}$ denotes the number of rounds experienced by user $t$. 
Given the reneging behavior modeled above, $s_{t}$ is the stopping time that represents the first time that the outcome $r_{t}^{(i)}$ is below the satisfaction level $\beta_t$, i.e. $s_{t} := \min\{i:r^{(i)}_{t}<\beta_{t}\}$.
Illustrative examples of heteroscedasticity and reneging risk are shown in Figure \ref{fig:illus}. 
In Figure \ref{fig:illus}(a), the variance of the outcome distribution gradually increases as the value of the one-dimensional context $x_{t,a}$ increases. 
Figure \ref{fig:illus}(b) shows the outcome distributions of the two actions for a user.
Specifically, the outcome distribution $P_{1}$ has mean $\mu_{1}$ and variance $\sigma^{2}_{1}$, and mean $\mu_{2}$ and variance $\sigma^{2}_{2}$ for $P_{2}$. 
As the two distributions correspond to the same user (but for different actions), they face the same satisfaction level $\beta$. 
In this example, the reneging risk $P_{2}(r<\beta)$ (the blue shaded area) is higher than $P_{1}(r<\beta)$ (the red shaded area). 

\vspace{-4mm}
\begin{figure}[htbp]
\centering 
$\begin{array}{c c}
    \multicolumn{1}{l}{\mbox{\bf }} & \multicolumn{1}{l}{\mbox{\bf }} \\
    \hspace{-5.0mm}\scalebox{0.23}{\includegraphics[width=\textwidth]{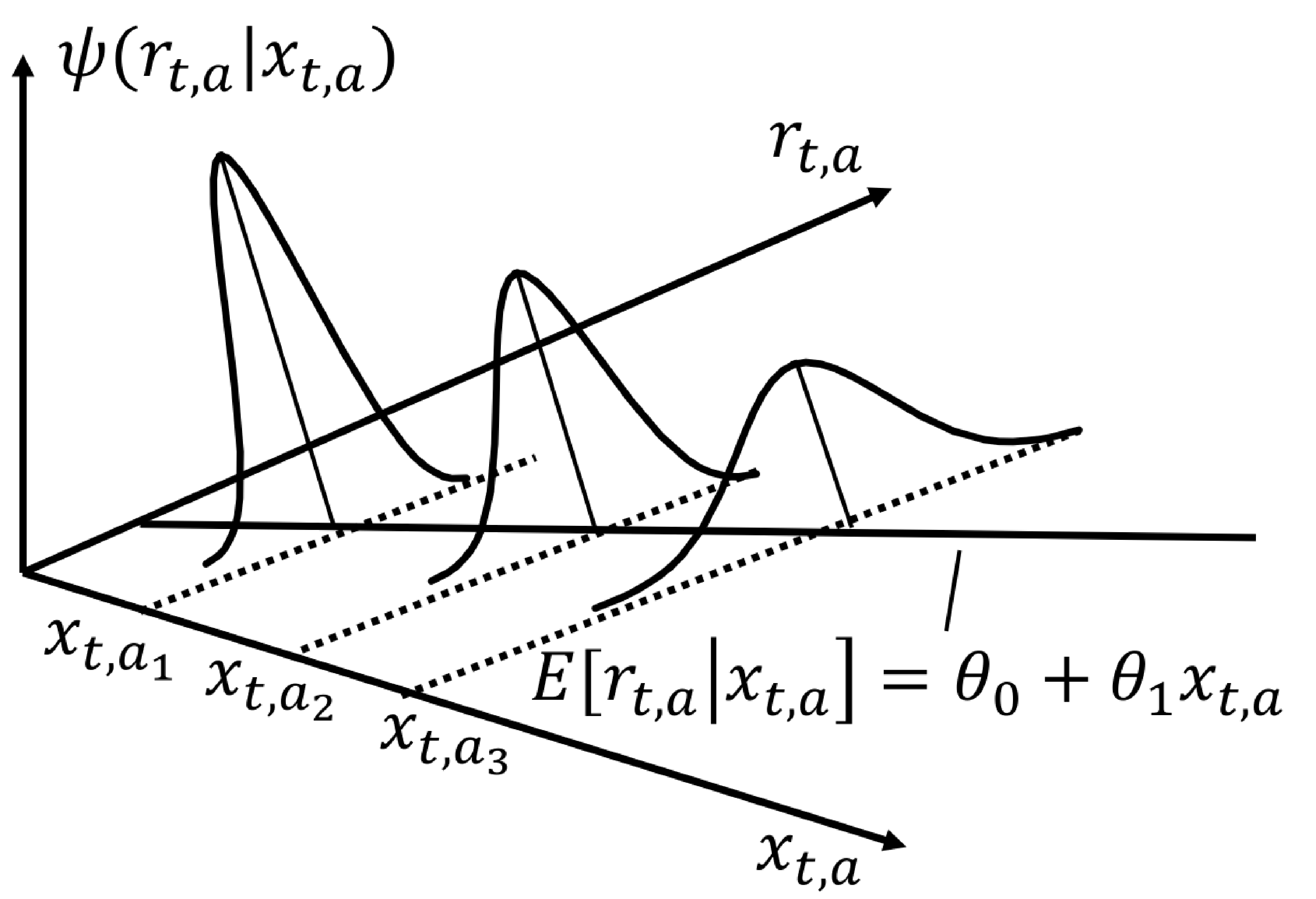}} & \hspace{-5.8mm} \scalebox{0.23}{\includegraphics[width=\textwidth]{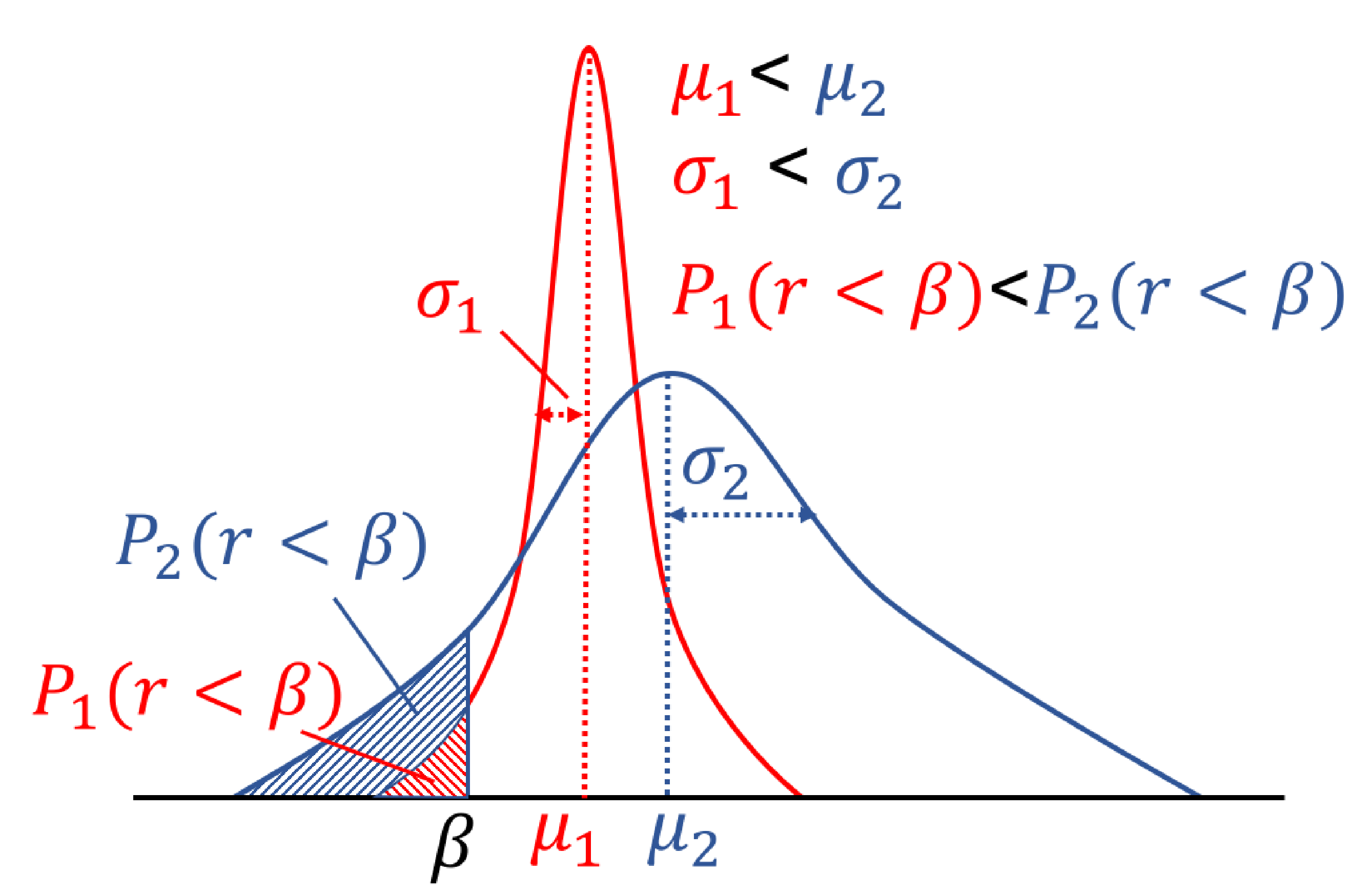}} \\
    \hspace{-2mm}\mbox{\scriptsize (a) Example of heteroscedasticity} & \hspace{-5mm} \mbox{\scriptsize (b) Example of reneging risk} \\ [-0.0cm]
    \end{array}$
\vspace{-2mm}
\caption{Illustrated examples for heteroscedasticity and reneging risk under presence of heteroscedasticity. ($\psi(\cdot)$ is the probability density function.)}
\label{fig:illus}
\vspace{-3mm}
\end{figure}

A policy $\pi \in \Pi$ is a rule for selecting an action at each round for a user based on the preceding interactions with that user and other users, where $\Pi$ denotes the set of all admissible policies. 
Let ${\pi}_{t}=\{{x}_{t,1},{x}_{t,2},\cdots\}$ denote the sequence of contexts that correspond to the actions for user $t$ under policy $\pi$.
Let $\overline{R}^{\pi}_t$ denote the expected lifetime of user $t$ under the action sequence $\pi_t$.
Then the total expected lifetime of $T$ users can be represented by
$\mathcal{R}^{\pi}(T) = \sum_{t=1}^{T}\overline{R}^{\pi_t}_{t}$.
Define $\pi^{*}$ as the optimal policy in terms of total expected lifetime among admissible policies, i.e.
$\pi^{*} = \argmax_{\pi\in\Pi}\mathcal{R}^{\pi}(T)$.
We are ready to define the \emph{pseudo regret} of the heteroscedastic linear bandits with reneging for a policy $\pi$ as
\begin{align}
    \text{Regret}_{T} := \mathcal{R}^{\pi^{*}}(T) -  \mathcal{R}^{\pi}(T).
\end{align}
The objective of the decision maker is to learn a policy that achieves as minimal a regret as possible.

\section{Algorithms and Results}
\label{section:alg}
In this section, we present a UCB-type algorithm for heteroscedastic linear bandits with reneging. 
We start by introducing general results on heteroscedastic regression.
\subsection{Heteroscedastic Regression}
\label{section: alg: regression}
In this section, we consider a general regression problem with heteroscedasticity.
\subsubsection{Generalized Least Squares Estimators}
\label{section: alg: estimator}
With a slight abuse of notation, let $\{(x_i,r_i)\in\mathbb{R}^{d}\times \mathbb{R}\}_{i=1}^{n}$ be a sequence of $n$ pairs of context and outcome that are realized by a user's actions.
Recall from (\ref{equation: reward distribution 1})-(\ref{equation: reward distribution 3}) that $r_{i}=\theta_{*}^{\top}x_{i}+\varepsilon(x_{i})$ and $\varepsilon(x_{i}) \sim \mathcal{N}\big(0,f(\phi_{*}^{\top}x_{i})\big)$ with unknown parameters $\theta_{*}$ and $\phi_{*}$.
Note that, given the contexts $\{x_i\}_{i=1}^{n}$, $\varepsilon(x_1),\cdots,\varepsilon(x_n)$ are mutually independent.
Let $r=(r_1,\cdots,r_n)^{\top}$ and $\varepsilon=(\varepsilon(x_1),\cdots,\varepsilon(x_n))$ be the row vectors of the $n$ outcome realizations and the deviations from the mean, respectively.
Let $\bm{X}_{n}$ be an $n\times d$ matrix in which the $i$-th row is $x_{i}^{\top}$, for all $1\leq i\leq n$.
We use $\widehat{{\theta}}_n, \widehat{{\phi}}_n\in\mathbb{R}^{d}$ to denote the estimators of $\theta_{*}$ and $\phi_{*}$ based on the observations $\{(x_i,r_i)\}_{i=1}^{n}$, respectively.
Moreover, define the estimated residual with respect to $\widehat{\theta}_{n}$ as
$\widehat{\varepsilon}(x_{i})={r_i}-{\widehat{\theta}_{n}}^{\top}x_{i}$.
Let $\widehat{\varepsilon}=(\widehat{\varepsilon}(x_{1}),\cdots,\widehat{\varepsilon}(x_{n}))^{\top}$.
Let $\bm{I}_d$ denote the $d\times d$ identity matrix, and let $z_1\circ z_2$ denote the Hadamard product of any two vectors $z_1,z_2$.
We consider the \emph{generalized least squares estimators} (GLSE) \cite{wooldridge2015introductory} as
\begin{align}
    \widehat{{\theta}}_n&=\big(\bm{X}^{\top}_{n}{\bm{X}_{n}}+\lambda \bm{I}_{d}\big)^{-1}\bm{X}_{n}^{\top}{r},\label{equation:theta estimator}\\
    \widehat{{\phi}}_n&=\big(\bm{X}^{\top}_{n}{\bm{X}_{n}}+\lambda \bm{I}_{d}\big)^{-1}\bm{X}^{\top}_{n}{f}^{-1}(\widehat{\varepsilon}\circ \widehat{\varepsilon}),\label{equation:phi estimator}
\end{align}
where $\lambda>0$ is some regularization parameter and $f^{-1}(\widehat{\varepsilon}\circ \widehat{\varepsilon})=(f^{-1}(\widehat{\varepsilon}(x_1)^{2}),\cdots,f^{-1}(\widehat{\varepsilon}(x_n)^{2}))^{\top}$ is the pre-image of the vector $\widehat{\varepsilon}\circ \widehat{\varepsilon}$.
\vspace{-3mm}

\begin{remark}
\normalfont Note that in (\ref{equation:theta estimator}), $\widehat{{\theta}}_n$ is the conventional ridge regression estimator.
On the other hand, to obtain an estimator $\widehat{{\phi}}_n$, (\ref{equation:phi estimator}) still follows the ridge regression approach, but with two additional steps: (i) derive the estimated \pch{residual} $\widehat{\varepsilon}$ based on $\widehat{\theta}_{n}$, and (ii) apply the map $f^{-1}(\cdot)$ on the square of $\widehat{\varepsilon}$.
\pch{Conventionally, GLSE is utilized to improve the efficiency of estimating $\theta_{*}$ under heteroscedasticity (e.g. Chapter 8.2 of \cite{wooldridge2015introductory}).
In our problem, we use GLSE to jointly learn $\theta_{*}$ and $\phi_{*}$ and thereby establish regret guarantees.}
However, it is not immediately clear how to obtain finite-time results regarding the confidence set for $\widehat{{\phi}}_n$.
This issue will be addressed in Section \ref{section: alg: confidence set}.
\end{remark}
\vspace{-1mm}

\subsubsection{Confidence Sets for GLSE}
\label{section: alg: confidence set}
In this section, we discuss the confidence sets for the estimators $\widehat{\theta}_{n}$ and $\widehat{\phi}_{n}$ described above.
To simplify notation, we define a $d\times d$ matrix $ {\bm{V}}_{n}$ as
\begin{equation}
     {\bm{V}}_{n} = \big(\bm{X}^{\top}_{n}{\bm{X}_{n}}+\lambda \bm{I}_{d}\big).\label{equation:Vn definition}
\end{equation}
A confidence set for $\widehat{\theta}_{t}$ was introduced in \cite{abbasi2011improved}. 
For convenience, we restate these elegant results in the following lemma.
\begin{lemma}{\normalfont{(Theorem 2 in \cite{abbasi2011improved})}}
\label{lemma:theta confidence}
For all $n\in\mathbb{N}$, define
\begin{equation}
    \alpha_n^{(1)}({\delta})=\sigma_{\max}^{2}\sqrt{d\log\Big(\frac{n+\lambda}{\delta\lambda}\Big)}+\lambda^{1/2}.
\end{equation}
For any $\delta>0$, we have
\begin{align}
    \mathbb{P}\Bigl\{\norm{\widehat{\theta}_n-\theta_{*}}_{ {\bm{V}_n}}\leq \alpha_n^{(1)}(\delta),\forall n\in\mathbb{N}\Bigr\}\geq 1-\delta,\label{equation:theta confidence set}
\end{align}
where $\norm{x}_{ {\bm{V}_n}}={\sqrt{x^{\top} {\bm{V}_n}x}}$ is the induced vector norm of vector $x$ with respect to $ {\bm{V}_n}$. 
\end{lemma}
\vspace{-2mm}

Next, we derive the confidence set for $\widehat{\phi}_{n}$.
Define 
\begin{align}
    \alpha^{(2)}(\delta)&=\sqrt{2d\pch{(\sigma_{\max}^{2})^{2}}\Big( \Big({\dfrac{1}{C_2}\ln({\dfrac{C_1}{\delta}}})\Big)^{2}+1\Big)},\\
    \alpha^{(3)}(\delta)&=\sqrt{2d\pch{\sigma_{\max}^{2}}\ln \big(\frac{d}{\delta}\big)},
\end{align}
where $C_1$ and $C_2$ are some universal constants that will be described in Lemma \ref{lemma:upper bound for first error term}.
The following is the main theorem on the confidence set for $\widehat{\phi}_{n}$.
\begin{theorem}
\label{theorem:phi confidence}
For all $n\in\mathbb{N}$, define
\begin{align}
    \rho_n({\delta})=\pch{\frac{1}{M_{f}}}\Big\{\alpha_{n}^{(1)}(\frac{\delta}{3})&\big(\alpha_{n}^{(1)}(\frac{\delta}{3})+2\alpha^{(3)}(\frac{\delta}{3})\big)\\
    &+\alpha^{(2)}(\frac{\delta}{3})\Big\}+\pch{L^{2}\lambda^{1/2}}.
\end{align}
For any $\delta>0$, with probability at least $1-2\delta$, we have
\begin{align}
    \norm{\widehat{\phi}_n-\phi_{*}}_{ {\bm{V}_n}}\leq \rho_n(\frac{\delta}{n^2})=\mathcal{O}\Big({\log(\frac{1}{\delta})+\log n}\Big), \forall n\in\mathbb{N}.\label{equation:phi confidence set}
\end{align}
\end{theorem}
\vspace{-6mm}

\begin{remark}
\normalfont 
\pch{As the estimator $\hat{\phi}_{n}$ depends on the residual term $\hat{\varepsilon}$, which involves the estimator $\hat{\theta}_{n}$, it is expected that the convergence speed of $\hat{\phi}_{n}$ would be no larger than that of $\hat{\theta}_{n}$.
Based on Theorem \ref{theorem:phi confidence} along with Lemma \ref{lemma:theta confidence}, we know that under GLSE, $\hat{\phi}_{n}$ converges to the true value at a slightly slower rate than $\hat{\theta}_{n}$.}
\end{remark}
\vspace{-1mm}

To demonstrate the main idea behind Theorem \ref{theorem:phi confidence}, we highlight the proof in the following Lemma \ref{lemma:inner product difference}-\ref{lemma:upper bound for third error term}.
We start by taking the inner products of an arbitrary vector $x$ with $\widehat{\phi}_{n}$ and ${\phi}_{*}$ to quantify the difference between $\widehat{\phi}_{t}$ and ${\phi}_{*}$.
\vspace{-2mm}

\begin{lemma}
\label{lemma:inner product difference}
For any $x\in\mathbb{R}^{d}$, we have
\begin{align}
    &\lvert x^{\top}\widehat{\phi}_{n}-x^{\top}\widehat{\phi}_{*}\rvert \leq \norm{x}_{{\bm{V}_n}^{-1}}\big\{\pch{\lambda}\norm{\phi_{*}}_{\bm{V}_n^{-1}}\label{equation:inner product difference 1}\\
    &\hspace{-9pt}+\norm{\bm{X}_{n}^{\top}\big(f^{-1}(\varepsilon\circ\varepsilon)- \bm{X}_{n}\phi_{*}\big)}_{{{ {\bm{V}}_{n}}^{-1}}}\label{equation:inner product difference 2}\\
    &\hspace{-9pt}+\pch{\frac{2}{M_{f}}}\norm{\bm{X}_{n}^{\top}\big(\varepsilon\circ \bm{X}_{n}(\theta_{*}-\widehat{\theta}_{n})\big)}_{{{ {\bm{V}}_{n}}^{-1}}}\label{equation:inner product difference 3}\\
    &\hspace{-9pt}+\pch{\frac{1}{M_{f}}}\norm{\bm{X}_{n}^{\top}\big(\bm{X}_{n}(\theta_{*}-\widehat{\theta}_{n})\circ \bm{X}_{n}(\theta_{*}-\widehat{\theta}_{n})\big)}_{{{ {\bm{V}}_{n}}^{-1}}}\big\}.\label{equation:inner product difference 4}
\end{align}
\end{lemma}
\vspace{-6mm}
\begin{proof}
The proof is provided in Appendix \ref{section:appendix:inner product difference}.
\end{proof}
\vspace{-4mm}


Based on Lemma \ref{lemma:inner product difference}, we provide upper bounds for the three terms in (\ref{equation:inner product difference 2})-(\ref{equation:inner product difference 4}) separately as follows.
\vspace{-2mm}

\begin{lemma}
\label{lemma:upper bound for first error term}
For any $n\in\mathbb{N}$, for any $\delta>0$, with probability at least $1-\delta$, we have
\begin{equation}
    \pch{M_f\norm{\bm{X}_{n}^{\top}\big(f^{-1}(\varepsilon\circ\varepsilon)- \bm{X}_{n}\phi_{*})\big)}_{{\inv{ {\bm{V}}_{n}}}}\leq {\alpha^{(2)}(\delta)}.}
\end{equation}
\end{lemma}
\vspace{-6mm}

\begin{proof}
We highlight the main idea of the proof. 
Recall that $\varepsilon(x_i)\sim \mathcal{N}(0,\phi_{*}^{\top}x_{i})$.
Therefore, $\varepsilon(x_i)^{2}$ is a $\chi_{1}^{2}$-distribution with a scaling of \pch{$f(\phi_{*}^{\top}x_{i})$}.
Hence, each element in \pch{$(f^{-1}(\varepsilon\circ\varepsilon)- \bm{X}_{n}\phi_{*})$} has zero mean.
Moreover, we observe that \pch{$\norm{\bm{X}_{n}^{\top}\big(f^{-1}(\varepsilon\circ\varepsilon)- \bm{X}_{n}\phi_{*})\big)}_{{\inv{ {\bm{V}}_{n}}}}$} is quadratic.
Since the $\chi_{1}^{2}$-distribution is sub-exponential, we utilize a proper tail inequality for quadratic forms of sub-exponential distributions to derive an upper bound.
The complete proof is provided in Appendix \ref{section:appendix:upper bound for first error term}.
\end{proof}
\vspace{-3mm}

Next, we derive an upper bound for (\ref{equation:inner product difference 3}).
\vspace{-2mm}

\begin{lemma}
\label{lemma:upper bound for second error term}
For any $n\in\mathbb{N}$, for any $\delta>0$, with probability at least $1-\delta$, we have
\begin{equation}
    \norm{\bm{X}_{n}^{\top}\big(\varepsilon\circ \bm{X}_{n}(\theta_{*}-\widehat{\theta}_{n})\big)}_{{\inv{ {\bm{V}}_{n}}}}\leq \alpha^{(1)}_{n}(\delta)\cdot\alpha^{(3)}(\delta).\label{equation: upper bound for second term}
\end{equation}
\end{lemma}
\vspace{-6mm}

\begin{proof}
The main challenge is that (\ref{equation: upper bound for second term}) involves the product of the residual $\varepsilon$ and the estimation error $\theta_{*}-\widehat{\theta}_{n}$.
Through some manipulation, we can decouple $\varepsilon$ from $\norm{\bm{X}_{n}^{\top}\big(\varepsilon\circ \bm{X}_{n}(\theta_{*}-\widehat{\theta}_{n})\big)}_{{\inv{ {\bm{V}}_{n}}}}$ and apply a proper tail inequality for quadratic forms of sub-Gaussian distributions.
The complete proof is provided in Appendix \ref{section:appendix:upper bound for second error term}.
\end{proof}
\vspace{-3mm}

Next, we provide an upper bound for (\ref{equation:inner product difference 4}).
\vspace{-2mm}

\begin{lemma}
\label{lemma:upper bound for third error term}
For any $n\in\mathbb{N}$, for any $\delta>0$, with probability at least $1-\delta$, we have
\begin{equation}
    \norm{\bm{X}_{n}^{\top}\big(\bm{X}_{n}(\theta_{*}-\widehat{\theta}_{n})\circ \bm{X}_{n}(\theta_{*}-\widehat{\theta}_{n})\big)}_{{\inv{ {\bm{V}}_{n}}}}\leq (\alpha^{(1)}_{n}(\delta))^{2}.\label{equation:upper bound for third term}
\end{equation}
\end{lemma}
\vspace{-6mm}
\begin{proof}
Since (\ref{equation:upper bound for third term}) does not involve $\varepsilon$, we can simply reuse the results in Lemma \ref{lemma:theta confidence} through some manipulation of (\ref{equation:upper bound for third term}).
The complete proof is provided in Appendix \ref{section:appendix:upper bound for third error term}.
\end{proof}
\vspace{-3mm}

Now, we are ready to prove Theorem \ref{theorem:phi confidence}.
\vspace{-3mm}

\begin{proof}[Proof of Theorem \ref{theorem:phi confidence}]
We use $\lambda_{\min}(\cdot)$ to denote the smallest eigenvalue of a square symmetric matrix.
Recall that $\bm{V}_{n}=\lambda\bm{I}_{d}+\bm{X}_{n}^{\top}\bm{X}_{n}$ is positive definite for all $\lambda>0$.
We have
\begin{equation}
    \norm{\phi_{*}}^{2}_{ {\bm{V}_n}^{-1}}\leq \norm{\phi_{*}}_{2}^{2}/\lambda_{\min}( {\bm{V}_n})\leq \norm{\phi_{*}}_{2}^{2}/\lambda\leq L^{2}/\lambda.\label{equation:norm of phi star}
\end{equation}
By (\ref{equation:norm of phi star}) and Lemmas \ref{lemma:inner product difference}-\ref{lemma:upper bound for third error term}, we know that for a given $n$ and a given $\delta_{n}>0$, with probability at least $1-\delta_{n}$, we have
\begin{align}
    &\lvert x^{\top}\widehat{\phi}_{n}-x^{\top}\widehat{\phi}_{*}\rvert\leq \norm{x}_{ {\bm{V}_n}^{-1}}\cdot \rho_n(\delta_n).\label{equation:inner product different short form}
\end{align}
Note that (\ref{equation:inner product different short form}) holds for any $x\in\mathbb{R}^{d}$.
By substituting $x= {\bm{V}}_{n}(\widehat{\phi}_{n}-\phi_{*})$ into (\ref{equation:inner product different short form}), we have
\begin{align}
    \norm{\widehat{\phi}_{n}-\phi_{*}}^{2}_{ {\bm{V}}_{n}}\leq \norm{ {\bm{V}}_{n}(\widehat{\phi}_{n}-\phi_{*})}_{ {\bm{V}_n}^{-1}}\cdot\rho_n(\delta_n).
\end{align}
Since $\norm{ {\bm{V}}_{n}(\widehat{\phi}_{n}-\phi_{*})}_{ {\bm{V}_n}^{-1}}=\norm{\widehat{\phi}_{n}-\phi_{*}}_{ {\bm{V}}_{n}}$, we know for a given $n$ and $\delta_n>0$, with probability at least $1-\delta_n$,
\begin{equation}
     \norm{\widehat{\phi}_{n}-\phi_{*}}_{ {\bm{V}}_{n}}\leq \rho_n(\delta_n).\label{equation:phi confidence bound for a single n} 
\end{equation}
Finally, to obtain a uniform bound, we simply choose $\delta_n=\delta/(n^{2})$ and apply the union bound to (\ref{equation:phi confidence bound for a single n}) over all $n\in\mathbb{N}$.
Note that \pch{$\sum_{n=1}^{\infty}\delta_n=\sum_{n=1}^{\infty}\delta/n^{2}=\frac{\pi^{2}}{6}{\delta}<{2\delta}$}.
Therefore, with probability at least \pch{$1-2\delta$}, for all $n\in\mathbb{N}$,
$\norm{\widehat{\phi}_{n}-\phi_{*}}_{ {\bm{V}}_{n}}\leq \rho_{n}\big(\frac{\delta}{n^{2}}\big)$.
\end{proof}
\vspace{-5mm}


\subsection{Heteroscedastic UCB Policy}
\label{section: alg: UCB}
In this section, we formally introduce the proposed policy based on the heteroscedastic regression in Section \ref{section: alg: regression}.
\subsubsection{An Oracle Policy}
\label{section: alg: UCB: oracle}
In this section, we consider a policy which has access to an oracle with full knowledge of $\theta_{*}$ and $\phi_{*}$.
Consider $T$ users that arrive sequentially.
Let ${\pi}^{\textrm{oracle}}_{t}=\{{x}_{t,1}^{*},{x}_{t,2}^{*},\cdots\}$ be the sequence of contexts that correspond to the actions for the user $t$ under an oracle policy $\pi^{\textrm{oracle}}$.
The oracle policy $\pi^{\textrm{oracle}}=\{\pi_{t}^{\textrm{oracle}}\}$ is constructed by choosing
\begin{equation}
    \pi^{\textrm{oracle}}_{t}= {\arg\max}_{\widetilde{x}_{t}=\{\widetilde{x}_{t,1},\widetilde{x}_{t,2}\cdots\}} R_{t}^{\widetilde{x}_{t}},\label{equation:oracle policy at t}  
\end{equation}
for each $t$.
Due to the construction in (\ref{equation:oracle policy at t}), we know that $\pi^{\textrm{oracle}}$ achieves the largest possible expected lifetime for each user $t$, and is hence optimal in terms of pseudo-regret defined in Section \ref{section:problem}.
By using an one-step optimality argument, it is easy to verify that $\pi^{\textrm{oracle}}$ is a fixed policy for each user $t$, i.e. ${x}_{t,i}={x}_{t,j}$, for all $i,j\geq 1$.
Let $\overline{R}_{t}^{*}$ denote the expected lifetime of user $t$ under $\pi^{\textrm{oracle}}$.
We have
\begin{equation}
    \overline{R}_{t}^{*}=\bigg({\Phi\Big(\frac{\beta_{t}-\theta_{*}^{\top}{x}_{t}^{*}}{\sqrt{f(\phi_{*}^{\top}{x}_{t}^{*})}}\Big)}\bigg)^{-1}.\label{equation:oracle per-user reward}
\end{equation}
Next, we derive a useful property regarding (\ref{equation:oracle per-user reward}). 
For any given ${\beta}\in [-B,\infty)$, define the function $h_{\beta}: [-1,1]\times [\sigma_{\min}^{2},\sigma_{\max}^{2}]\rightarrow \mathbb{R}$ as
\begin{equation}
    h_{\beta}(u,v)=\bigg({\Phi\Big(\frac{\beta-u}{\sqrt{f(v)}}\Big)}\bigg)^{-1}.\label{equation:hx definition}
\end{equation}
\vspace{-3mm}

Note that for any given $x\in\mathcal{X}$, $h_{\beta}({\theta_{*}}^{\top}x,{\phi_{*}}^{\top}x)$ equals the expected lifetime of a single user with threshold $\beta$ if a fixed action with context $x$ is chosen under parameters $\theta_*,\phi_*$. 
We show that $h_{\beta}(\cdot,\cdot)$ has the following nice property.

\vspace{-2mm}
\begin{theorem}
\label{theorem:h beta linear approximation}
Let $\bm{M}$ be a $d\times d$ invertible matrix. 
For any $\theta_1,\theta_2\in\mathbb{R}^{d}$ with $\norm{\theta_1}\leq 1$, $\norm{\theta_2}\leq 1$, for any $\phi_1,\phi_2\in\mathbb{R}^{d}$ with $\norm{\phi_1}\leq L$, $\norm{\phi_2}\leq L$, for any $\beta\in[-B,\infty)$, $\forall x\in\mathcal{X}$,
\begin{align}
    &h_{\beta}\big(\theta_2^{\top}x, \phi_2^{\top}x\big)-h_{\beta}\big(\theta_1^{\top}x, \phi_1^{\top}x\big)\leq\label{equation:h beta approx theorem 1}\\
    &\hspace{0pt}\Big(C_{3}\norm{\theta_{2}-\theta_{1}}_{\bm{M}}+C_{4}\norm{\phi_{2}-\phi_{1}}_{\bm{M}}\Big)\cdot\norm{x}_{\bm{M}^{-1}},\label{equation:h beta approx theorem 2}
\end{align}
where $C_3$ and $C_4$ are some finite positive constants that are independent of $\theta_1,\theta_2, \phi_1,\phi_2,$ and $\beta$.
\end{theorem}
\vspace{-6mm}

\begin{proof}
The main idea is to apply first-order approximation under Lipschitz continuity of $h_{\beta}(\cdot,\cdot)$. The detailed proof is provided in Appendix \ref{section:appendix:h beta linear approximation}.
\end{proof}
\vspace{-3mm}

\subsubsection{The HR-UCB Policy}
To begin with, we introduce an upper confidence bound based on the GLSE described in Section \ref{section: alg: regression}.
Note that the results in Theorem \ref{theorem:phi confidence} depend on the size of the set of context-outcome pairs.
Moreover, in our bandit model, the number of rounds of each user is a stopping time and can be arbitrarily large.
To address this, we propose to actively maintain a \emph{regression sample set} $\mathcal{S}$ through a function $\Gamma(t)$.
Specifically, we let the size of $\mathcal{S}$ grow at a proper rate regulated by $\Gamma(t)$. 
One example is to choose $\Gamma(t)=Kt$ for some constant $K\geq 1$. 
Since each user will play for at least one round, we know $\abs{\mathcal{S}}$ is at least $t$ after interacting with $t$ users.
We use $\mathcal{S}(t)$ to denote the regression sample set right after the departure of user $t$.
Moreover, let $\bm{X}_{t}$ be the matrix in which the rows are composed by the contexts of all the elements in $\mathcal{S}(t)$.
Similar to (\ref{equation:Vn definition}), we define $\bm{V}_t=\bm{X}^{\top}_{t}{\bm{X}_{t}}+\lambda \bm{I}_{d}$, for all $t\geq 1$.
To simplify notation, we also define
\begin{equation}
    \xi_{t}(\delta):=C_{3}\alpha^{(1)}_{\abs{\mathcal{S}(t)}}(\delta)+C_{4}\rho_{\abs{\mathcal{S}(t)}}(\delta/\abs{\mathcal{S}(t)}^{2}).
\end{equation}
For any $x\in\mathcal{X}$, we define the upper confidence bound as
\begin{align}
    {Q}^{\textrm{HR}}_{t+1}({x})&=\pch{h_{\beta_{t+1}}}\big(\widehat{\theta_t}^{\top}x, \widehat{\phi_t}^{\top}x)+\xi_t(\delta)\cdot\norm{{x}}_{ {\bm{V}}^{-1}_{t}}.\label{equation:xi definition}
\end{align}
\xl{Note that the exploration over users, guaranteeing sublinear regret under heteroscedasticity, is handled by encoding the confidence bound in $Q_{t}^{\textrm{HR}}$ so that later users with similar contexts are treated differently.}
Next, we show that $Q_{t}^{\textrm{HR}}({x})$ is indeed an upper confidence bound.
\vspace{-1mm}
\begin{algorithm}[!htbp]
  \caption{The HR-UCB Policy}
  \label{alg:UCB}
  \begin{algorithmic}[1]
  \STATE {$\mathcal{S}\leftarrow \varnothing$, action set $\mathcal{A}$, function $\Gamma(t)$, and $T$}
  \FOR{\normalfont{each user} $t=1,2,\cdots, T$}
    \STATE observe $x_{t,a}$ for all $a\in\cal{A}$ and reset $i\leftarrow 1$
    \WHILE{\normalfont{user $t$ stays}}
        \STATE $\pi_t^{(i)}=\arg\max_{{x}_{t,a}\in\mathcal{X}_{t}}{Q}^{\textrm{HR}}_{t}(x_{t,a})$ (ties are broken arbitrarily)
        \STATE apply the action $\pi_{t}^{(i)}$ and observe the outcome $r_{t}^{(i)}$ and if the reneging event occurs
        \IF{$\abs{\mathcal{S}}<\Gamma(t)$}
            \STATE $\mathcal{S}\leftarrow \mathcal{S}\cup \{(x_{t,\pi_{t}^{(i)}},r_{t}^{(i)})\}$
        \ENDIF
        \STATE $i\leftarrow i+1$
    \ENDWHILE
    \STATE update $\widehat{\theta_t}$ and $\widehat{\phi_t}$ by (\ref{equation:theta estimator})-(\ref{equation:phi estimator}) based on $\mathcal{S}$\label{lst:update theta and phi}
  \ENDFOR
\end{algorithmic}
\end{algorithm}
\vspace{-3mm}

\begin{lemma}
\label{lemma:Q is an upper confidence bound}
If the confidence set conditions (\ref{equation:theta confidence set}) and (\ref{equation:phi confidence set}) are satisfied, then for any $x\in\mathcal{X}$,
\begin{equation*}
    0\leq Q^{\emph{HR}}_{t+1}(x)-\pch{h_{\beta_{t+1}}}\big(\theta_{*}^{\top}x, \phi_{*}^{\top}x)\leq 2\xi_{t}(\delta)\norm{x}_{ {\bm{V}}_{t}^{-1}}.
\end{equation*}
\end{lemma}
\vspace{-6mm}
\begin{proof}
The proof is provided in Appendix \ref{section:appendix:Q is an upper confidence bound}.
\end{proof}
\vspace{-4mm}

Now, we formally introduce the HR-UCB algorithm.
\vspace{-4mm}

\begin{itemize}[leftmargin=*]
    \item For each user $t$, HR-UCB observes the contexts of all available actions and then chooses an action based on the indices $Q_{t}^{\textrm{HR}}$ that depend on $\widehat{\theta}_t$ and $\widehat{\phi}_t$. To derive these estimators by (\ref{equation:theta estimator}) and (\ref{equation:phi estimator}), HR-UCB actively maintains a sample set $\mathcal{S}$, whose size is regulated by a function $\Gamma(t)$.
    \vspace{-2mm}
    \item After applying an action, HR-UCB observes the corresponding outcome and the reneging event if any. The current context-outcome pair will be added to $\mathcal{S}$ only if the size of $\mathcal{S}$ is less than $\Gamma(t)$.
    \vspace{-2mm}
    \item Based on the regression sample set $\mathcal{S}$, HR-UCB updates $\widehat{\theta}_t$ and $\widehat{\phi}_t$ right after the departure of each user.
\end{itemize}
\vspace{-2mm}

The complete algorithm is shown in Algorithm \ref{alg:UCB}.

\vspace{-3mm}
\begin{remark}
\normalfont \pch{In Algorithm \ref{alg:UCB}}, $\widehat{\theta}_{t}$ and $\widehat{\phi}_{t}$ are updated right after the departure of each user.
Alternatively, $\widehat{\theta}_{t}$ and $\widehat{\phi}_{t}$ can be updated whenever $\mathcal{S}$ is updated.
While this alternative makes slightly better use of the observations, it also incurs more computation overhead. 
For ease of exposition, we focus on the "lazy-update" version presented in Algorithm \ref{alg:UCB}.
\end{remark}
\vspace{-3mm}

\subsection{Regret Analysis}
\label{section:alg:analysis}
In this section, we provide regret analysis for HR-UCB.

\vspace{-3mm}
\begin{theorem}
\label{theorem:regret}
Under HR-UCB, with probability at least $1-\delta$, the pseudo regret is upper bounded as
\begin{align}
    &\emph{Regret}_{T}\leq \sqrt{8\pch{\xi_{T}^{2}\big(\frac{\delta}{3}\big)}T\cdot d\log\Big(\frac{T+\lambda d}{\pch{\lambda d}}\Big)}\label{equation:regret_upper}\\
    &=\mathcal{O}\bigg(\sqrt{T\log \pch{\Gamma(T)}\cdot\Big(\log\big(\Gamma(T)\big)+\log(\frac{1}{\delta})\Big)^{2}}\bigg).\label{equation:regret}
\end{align}
By choosing $\Gamma(T)=KT$ with a constant $K>0$, we have
\begin{equation}
    \emph{Regret}_{T}=\mathcal{O}\bigg(\sqrt{T\log T\cdot\Big(\log T+\log(\frac{1}{\delta})\Big)^{2}}\bigg).
\end{equation}
\end{theorem}
\vspace{-6mm}

\begin{proof}
The proof is provided in Appendix \ref{section:appendix:regret}.
\end{proof}
\vspace{-3mm}

Theorem \ref{theorem:regret} presents a high-probability regret bound.
To derive an expected regret bound, we can set $\delta=1/T$ in (\ref{equation:regret}) and get ${{\mathcal{O}}(\sqrt{T(\log T)^{3}})}$.
Also note that the upper bound (\ref{equation:regret_upper}) depends on $\sigma_{\max}$ only through the pre-constant of $\xi_T$. 

\vspace{-8mm}
\xl{\begin{remark}
\normalfont A policy that always assumes $\sigma_{\max}$ as variance tends to choose the action with the largest mean reward since it implies a smaller reneging probability. As a result, such type of policy incurs linear regret. This will be further demonstrated via simulations in Section \ref{section:simulation}.
\end{remark}}

\vspace{-10mm}
\xl{
\begin{remark}
\normalfont The regret proof still goes through for sub-Gaussian noise by (a) reusing the same sub-exponential concentration inequality in Lemma \ref{section:appendix:inner product difference} since the square of an sub-Gaussian distribution is sub-exponential, (b) replacing the Gaussian concentration inequality in Lemma \ref{section:appendix:upper bound for second error term} with a sub-Gaussian one, and (c) deriving ranges of the first two derivatives of sub-Gaussian CDF.
\end{remark}
}
\vspace{-3mm}

\xl{
\noindent \textbf{Remark 6} The assumption that $\beta_t$ is known can be relaxed to the case where only the distribution of $\beta_t$ is known. The analysis can be adapted to this case by (a) rewriting the reneging probability in (\ref{equation: probability}) and $h_\beta(u,v)$ in (\ref{equation:hx definition}) via integration over distribution of $\beta_t$, (b) deriving the corresponding expected lifetime under oracle policy in (\ref{equation:oracle per-user reward}), and (c) reusing Theorem \ref{theorem:phi confidence} and Lemma \ref{lemma:theta confidence} as the GLSE does not rely on the knowledge of $\beta_t$.
}
\vspace{-1mm}

\noindent \textbf{Remark 7} We briefly discuss the difference between our regret bound and the regret bounds of other related settings.
Note that if the satisfaction level $\beta_t=\infty$ for all $t$, then all the users will quit after exactly one round.
This corresponds to the conventional contextual bandits setting (e.g. homoscedastic case \cite{chu2011contextual} and heteroscedastic case \cite{kirschner2018information}).
In this degenerate case, our regret bound is $\mathcal{O}(\sqrt{T(\log T)}\cdot\log T)$, which has an additional factor $\log T$ resulting from the heteroscedasticity.

\vspace{0mm}

\section{Simulation Results}
\label{section:simulation}
\vspace{0mm}

In this section, we evaluate the performance of HR-UCB.
We consider 20 actions available to the decision maker. 
For simplicity, the context of each user-action pair is designed to be a four-dimensional vector, which is drawn uniformly at random from a unit ball.
For the mean and variance of the outcome distribution, we set $\theta_{*}=[0.6, 0.5, 0.5, 0.3]^{\top}$ and $\phi_{*}=[0.5, 0.2, 0.8, 0.9]^{\top}$, respectively.
We consider the function $f(x)={x + L}$ with $L=2$ and $M_{f}=1$. 
The acceptance level of each user is drawn uniformly at random from the interval $[-1,1]$.
We set $T=30000$ throughout the simulations.
For HR-UCB, we set $\delta = 0.1$ and $\lambda=1$.
All the results in this section are the average of 20 simulation trials.
Recall that $K$ denotes the growth rate of the regression sample set for HR-UCB. We start by evaluating the pseudo regrets of HR-UCB under different $K$, as shown in Figure \ref{fig:HR-UCB regret}.
Note that HR-UCB achieves a sublinear regret regardless of $K$.
The effect of $K$ is only reflected when the number of users is small.
Specifically, a smaller $K$ induces a slightly higher regret since it requires more users in order to accurately learn the parameters.
Based on Figure \ref{fig:HR-UCB regret}, we set $K=5$ for the rest of the simulations.

Next, we compare the HR-UCB policy with the well-known LinUCB policy \cite{li2010contextual} and the CMDP policy~\cite{modi2018markov}. Figure~\ref{fig:regret comparison} shows the pseudo regrets under LinUCB, CMDP and HR-UCB. LinUCB achieves a linear regret because it does not take into account the heteroscedasticity of the outcome distribution in the existence of reneging. For each user, LinUCB simply chooses the action with the largest predicted mean of the outcome distribution. 
\xl{The regret attained by CMDP policy also appears linear. This is because CMDP handles contexts by partitioning the context space and then learning each partition-induced MDP separately. 
Due to the continuous context space, the CMDP policy requires numerous partitions as well as plentiful exploration for all MDPs. 
To make the comparison more fair, we consider a simpler setting with a discrete context space of size 10 and only 2 actions (with other parameters unchanged). In this setting, Figure~\ref{fig:comparison_hrucb_emdp_new} shows that the regret attained by CMDP is still much larger than that by HR-UCB and thereby shows the advantage of the proposed solution.}


\xl{As mentioned in Section~\ref{section:alg:analysis}, a policy (denoted by $\sigma_{\max}$-UCB) that always assumes $\sigma_{\max}$ as variance tends to choose the action with the largest mean and thus incurs linear regret. We demonstrate the statement in experiments shown by Figure~\ref{fig:comparison_sgimamax_hrucb_new}, where the $\sigma_{\max}$-UCB policy attains a linear regret vs. HR-UCB achieves a sublinear and much smaller regret.} Through simulations, we validate that HR-UCB achieves the regret performance as discussed in Section \ref{section:alg}.

\vspace{-6mm}

\begin{figure}[!htbp]
    \centering
    \subfloat[Pseudo regrets: HR-UCB with different $K$.]{
    \includegraphics[width=0.22\textwidth]{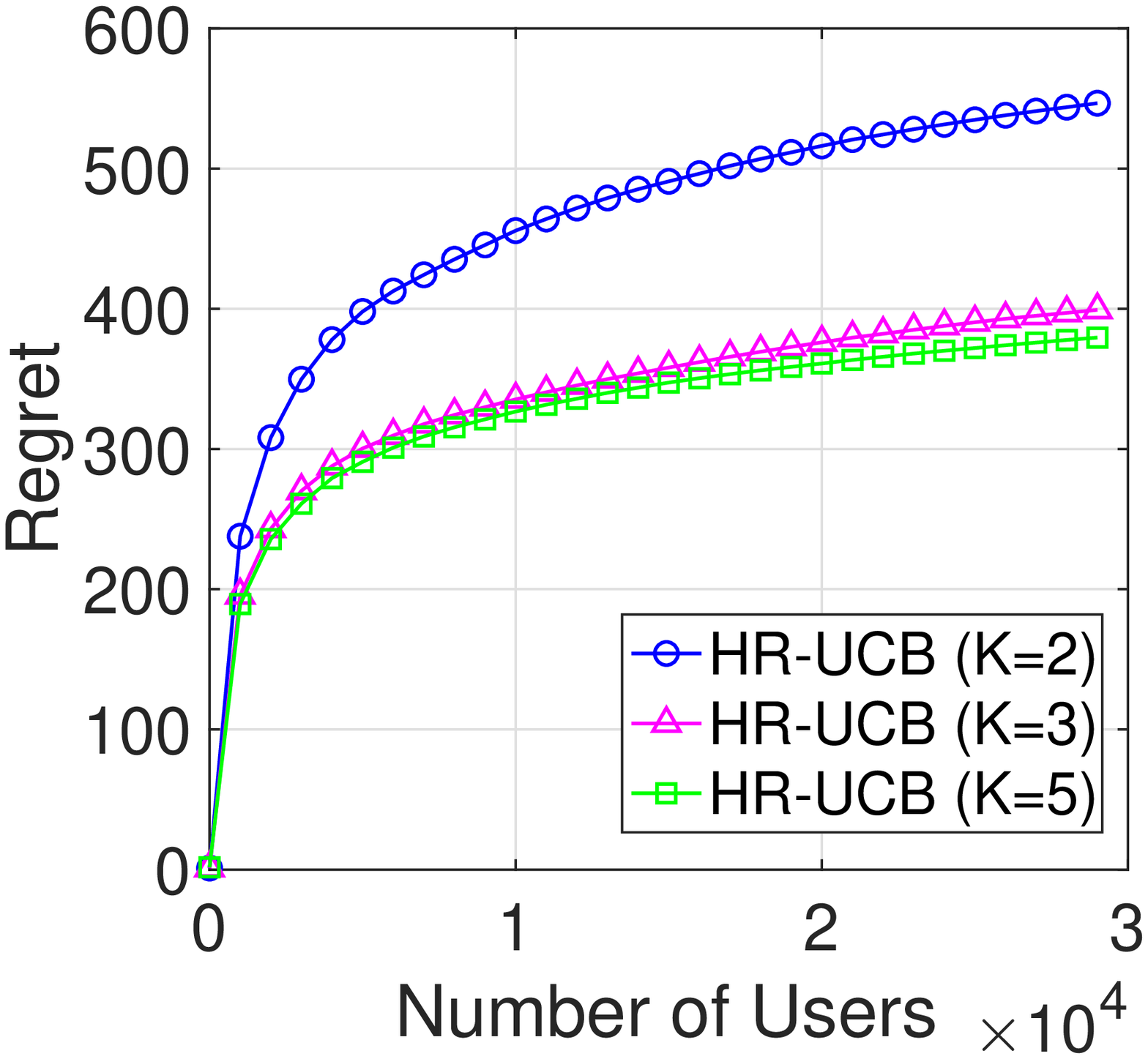}
    \label{fig:HR-UCB regret}}
    \subfloat[Pseudo regrets: LinUCB, CMDP and HR-UCB ($K=5$).]{
    \includegraphics[width=0.22\textwidth]{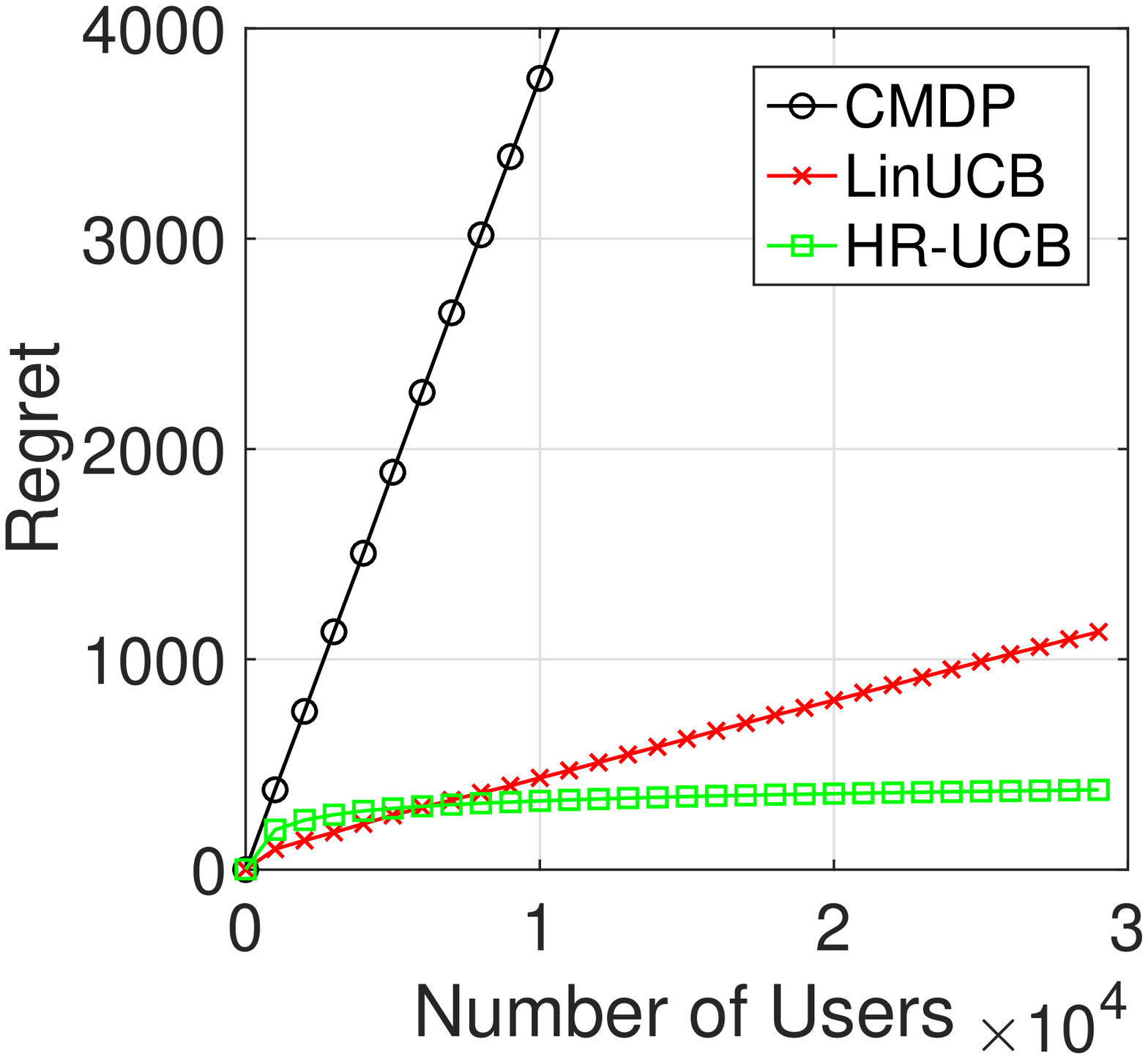}
    \label{fig:regret comparison}}\\[-3mm]
    \subfloat[Pseudo regrets: $\sigma_{\max}$-UCB and HR-UCB ($K=5$).]{
    \includegraphics[width=0.22\textwidth]{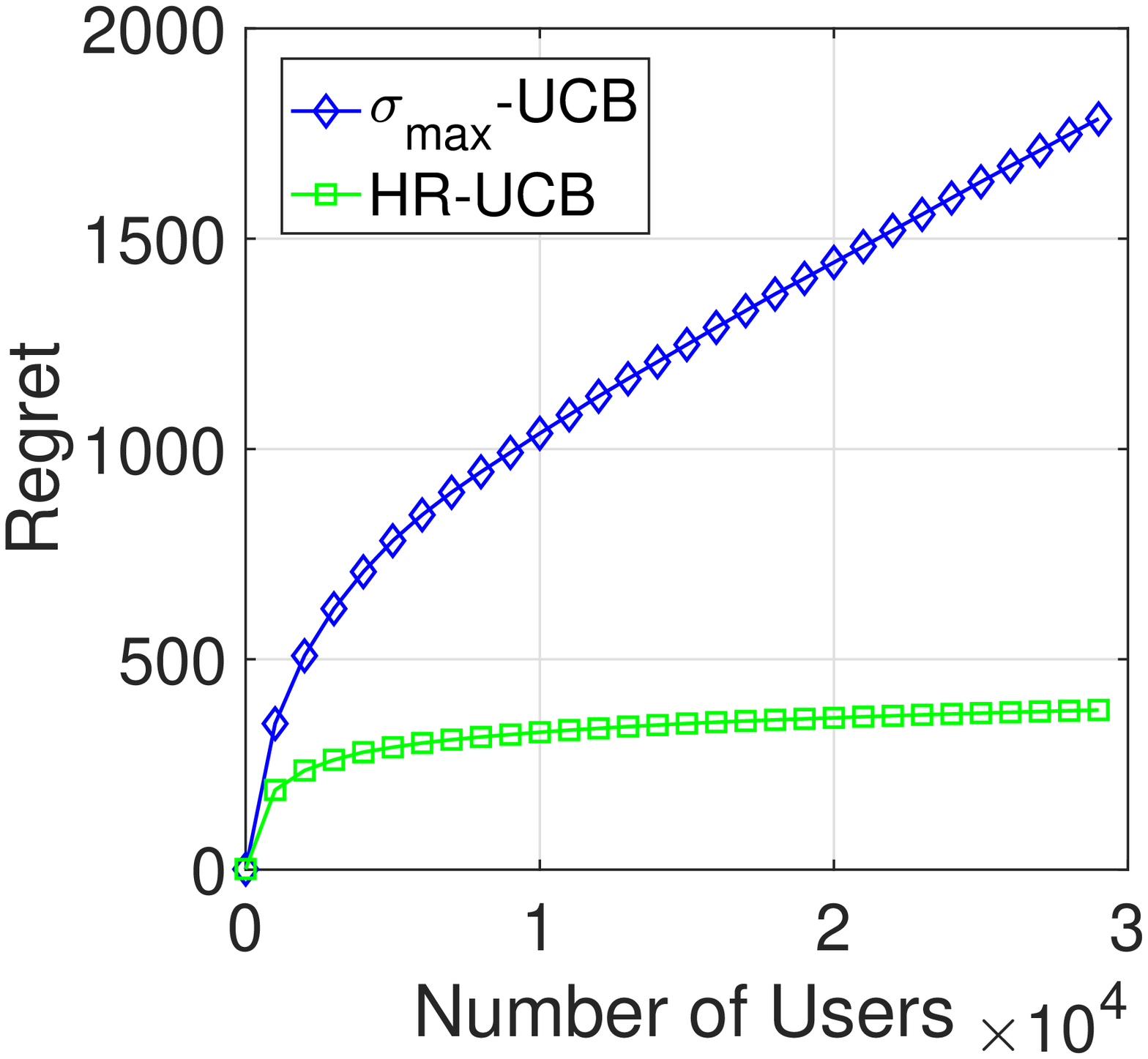}
    \label{fig:comparison_sgimamax_hrucb_new}}
    \subfloat[Pseudo regrets: CMDP and HR-UCB ($K=5$).]{
    \includegraphics[width=0.22\textwidth]{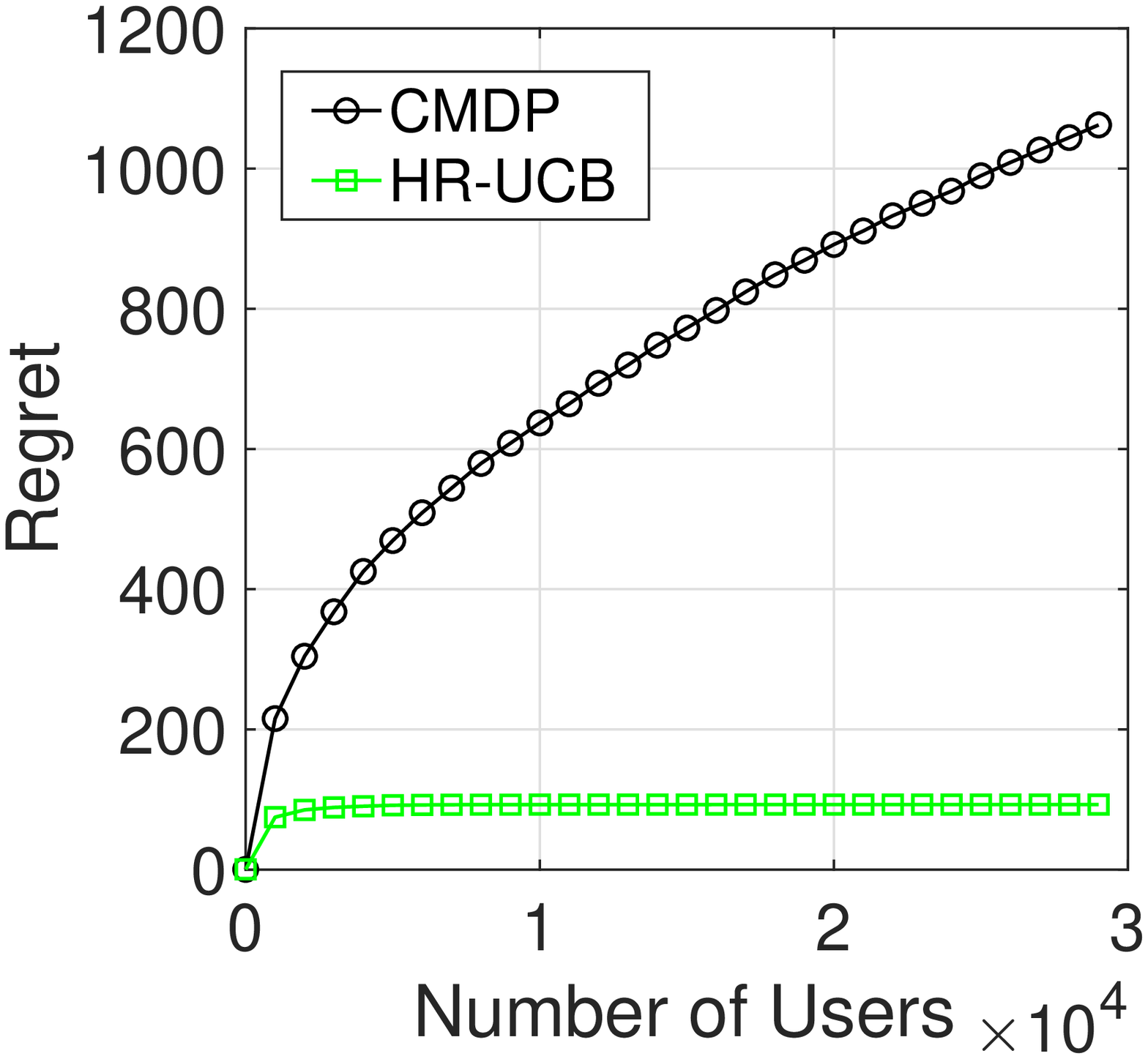}
    \label{fig:comparison_hrucb_emdp_new}}
    \vspace{-2mm}
    \caption{Comparison of pseudo regrets.}
\end{figure}
\vspace{-3mm}

\vspace{-2mm}

\section{Concluding Remarks} \label{section:concluding_remarks}
\vspace{0mm}

There are several ways to extend the studies in this paper. First, the techniques used to estimate heteroscedastic variance and establishing sub-linear regret under the presence of heteroscedasticity can be extended to other variance-sensitive bandit problems, e.g., risk-averse bandits and thresholding bandits. \xl{Second, the studies can be easily adapted to another objective - maximizing total collected rewards by: (a) replacing $h_{\beta}(u,v)$ in (\ref{equation:oracle per-user reward}) with $\hat{h}_{\beta}(u,v)=u\cdot h_{\beta}(u,v)$, (b) reusing Theorem \ref{theorem:phi confidence} and Lemma \ref{lemma:theta confidence}, and (c) making minor changes to constants $C_3$, $C_4$ in (\ref{equation:h beta approx theorem 2}). Third, another promising extension is to use active-learning to update sample set $\mathcal{S}$~\cite{riquelme2017online}.
To provide theoretical guarantees, these active-learning approaches often assume that arriving contexts are i.i.d. In contrast, since that assumption can be easily invalid (e.g., it is adversarial), we establish the regret bound without making any such assumption. Finally, in this paper, the problem of knowledge transfer across users is given more importance than learning for a single user. This is because, compared to the population of potential users, a user's lifetime is mostly short. Therefore, another possible extension is to take into account the exploration during the lifetime of each individual user.}

\section*{Acknowledgements}

\xl{This material is based upon work partially supported by NSF under Science \& Technology Center Grant CCF-0939370, and Texas A\&M University under the President’s Excellence Funds X Grants Program. We would like to thank all reviewers and Dr. P. S. Sastry for their insightful suggestions!}


\bibliography{reference}
\bibliographystyle{icml2019}

\appendix
\clearpage
\twocolumn[
\icmltitle{{Supplementary Material:}\\Stay With Me: Lifetime Maximization Through \\ Heteroscedastic Linear Bandits With Reneging}
]
\section{Appendix}
\label{section:appendix}

\subsection{Proof of Lemma \ref{lemma:inner product difference}}
\label{section:appendix:inner product difference}
\begin{proof}
Recall that $\bm{V}_n=\big(\bm{X}^{\top}_{n}{\bm{X}_{n}}+\lambda \bm{I}_{d}\big)$. Note that
\begin{align}
    \widehat{\phi}_{n}&=(\bm{X}_{n}^{\top}\bm{X}_{n}+\lambda I_{d})^{-1}\bm{X}_{n}^{\top}f^{-1}(\widehat{\varepsilon}\circ\widehat{\varepsilon})\\
    &=\bm{V}_n^{-1}\bm{X}_{n}^{\top}f^{-1}(\widehat{\varepsilon}\circ\widehat{\varepsilon})\\
    &=\bm{V}_n^{-1}\bm{X}_{n}^{\top}\big(f^{-1}(\widehat{\varepsilon}\circ\widehat{\varepsilon})-\bm{X}_{n}\phi_{*}+\bm{X}_{n}\phi_{*}\big)\\
    &\hspace{20pt}+\lambda \bm{V}_n^{-1}\phi_{*}-\lambda \bm{V}_n^{-1}\phi_{*}\\
    &=\bm{V}_n^{-1}\bm{X}_{n}^{\top}\big(f^{-1}(\widehat{\varepsilon}\circ\widehat{\varepsilon})-\bm{X}_{n}\phi_{*}\big)-\lambda \bm{V}_n^{-1}\phi_{*}+\phi_{*}.
\end{align}
Therefore, for any $x\in\mathbb{R}^{d}$, we know
\begin{align}
    &\lvert x^{\top}\widehat{\phi}_{n}-x^{\top}\widehat{\phi}_{*}\rvert\label{equation:inner product difference 1 in appendix}\\
    &=\lvert x^{\top}\bm{V}_n^{-1}\bm{X}_{n}^{\top}\big(f^{-1}(\widehat{\varepsilon}\circ\widehat{\varepsilon})-\bm{X}_{n}\phi_{*}\big)-\lambda x^{\top}\bm{V}_n^{-1}\phi_{*}\rvert\label{equation:inner product difference 2 in appendix}\\
    &\leq \norm{x}_{{\bm{V}_n}^{-1}}\Big(\lambda\norm{\phi_{*}}_{\pch{\bm{V}_n^{-1}}}\\
    &\hspace{36pt}+\norm{\bm{X}_{n}^{\top}\big(f^{-1}(\widehat{\varepsilon}\circ\widehat{\varepsilon})- \bm{X}_{n}\phi_{*})\big)}_{{{ {\bm{V}}_{n}}^{-1}}}\Big).\label{equation:inner product difference 3 in appendix}
\end{align}
Moreover, by rewriting $\widehat{\varepsilon}=\widehat{\varepsilon}-\varepsilon+\varepsilon$, we have
\begin{align}
    &f^{-1}(\widehat{\varepsilon}\circ\widehat{\varepsilon})\label{equation:epsilon hat 1}\\
    &=f^{-1}\big(({\widehat{\varepsilon}-\varepsilon+\varepsilon})\circ({\widehat{\varepsilon}-\varepsilon+\varepsilon})\big)\label{equation:epsilon hat 2}\\
    &\pch{= f^{-1}({\varepsilon}\circ{\varepsilon})+M_{f}^{-1}\Big(2\big(\varepsilon\circ \bm{X}_{n}(\theta_{*}-\widehat{\theta}_{n})\big)}\label{equation:epsilon hat 3}\\
    &\hspace{36pt}+\big(\bm{X}_{n}(\theta_{*}-\widehat{\theta}_{n})\circ \bm{X}_{n}(\theta_{*}-\widehat{\theta}_{n})\big)\Big),\label{equation:epsilon hat 4}
\end{align}
where (\ref{equation:epsilon hat 3})-(\ref{equation:epsilon hat 4}) follow from the fact that \pch{both $f(\cdot)$ and $f^{-1}(\cdot)$ are linear with a slope $M_f$ and $M_f^{-1}$, respectively, as described in Section \ref{section:problem}}. 
Therefore, by (\ref{equation:inner product difference 1 in appendix})-(\ref{equation:epsilon hat 4}) and the Cauchy-Schwarz inequality, we have
\begin{align}
    &\lvert x^{\top}\widehat{\phi}_{n}-x^{\top}\widehat{\phi}_{*}\rvert\leq \norm{x}_{{\bm{V}_n}^{-1}}\Big\{\lambda\norm{\phi_{*}}_{\bm{V}_n^{-1}}\label{equation:inner product difference 4 in appendix}\\
    &\hspace{20pt}+\norm{\bm{X}_{n}^{\top}\big(f^{-1}(\varepsilon\circ\varepsilon)- \bm{X}_{n}\phi_{*})\big)}_{{{ {\bm{V}}_{n}}^{-1}}}\label{equation:inner product difference 5 in appendix}\\
    &\hspace{20pt}+\pch{2M_{f}^{-1}}\norm{\bm{X}_{n}^{\top}\big(\varepsilon\circ \bm{X}_{n}(\theta_{*}-\widehat{\theta}_{n})\big)}_{{{ {\bm{V}}_{n}}^{-1}}}\label{equation:inner product difference 6 in appendix}\\
    &\hspace{20pt}+\pch{M_{f}^{-1}}\norm{\bm{X}_{n}^{\top}\big(\bm{X}_{n}(\theta_{*}-\widehat{\theta}_{n})\circ \bm{X}_{n}(\theta_{*}-\widehat{\theta}_{n})\big)}_{{{ {\bm{V}}_{n}}^{-1}}}\Big\}.\label{equation:inner product difference 7 in appendix} 
\end{align}
\end{proof}

\subsection{Proof of Lemma \ref{lemma:upper bound for first error term}}
\label{section:appendix:upper bound for first error term}
We first introduce the following useful lemmas.
\begin{lemma}[Lemma 8.2 in \cite{erdHos2012bulk}]
\label{lemma:Erdos lemma}
Let $\{a_{i}\}_{i=1}^{N}$ be $N$ independent random complex variables with zero mean and variance $\sigma^{2}$ and having uniform sub-exponential decay, i.e., there exists $\kappa_1,\kappa_2>0$ such that
\begin{align}
    {\mathbb{P}}\{\abs{a_{i}}\geq x^{\kappa_1}\}\leq \kappa_2 e^{-x}.
\end{align}
We use $a^\emph{{H}}$ to denote the conjugate transpose of $a$.
Let $a=(a_1,\cdots,a_N)^{\top}$, let $\xbar{a_i}$ denote the complex conjugate of $a_i$, for all $i$, and let $\bm{B}=(B_{ij})$ be a complex $N\times N$ matrix.
Then, we have
\begin{align}
    &{\mathbb{P}}\Big\{\abs{a^{\emph{H}}\bm{B}a-\sigma^{2}\emph{tr}(\bm{B})}\geq s \sigma^{2}\Big(\sum_{i=1}^{N}\abs{{B}_{ii}}^{2}\Big)^{-1/2}\Big\}\\
    &\hspace{24pt}\leq C_1 \emph{exp}\Big(-C_2\cdot s^{{1}/{(1+\kappa_1)}}\Big),
\end{align}
where $C_1$ and $C_2$ are positive constants that depend only on $\kappa_1,\kappa_2$. 
Moreover, for the standard $\chi_{1}^{2}$-distribution, $\kappa_1=1$ and $\kappa_2=2$. 
\end{lemma}
For any $p\times q$ matrix $\bm{A}$, we define the induced matrix norm as $\norm{\bm{A}}_2:=\max_{v\in{\mathbb{R}}^{q},\norm{v}_{2}=1} \norm{\bm{A}v}_{2}$.
\begin{lemma}
\label{lemma:norm of V -1/2 X transpose}
\begin{equation}
    \norm{  {\bm{V}_n}^{-1/2}\bm{X}^{\top}}_{2}\leq 1, \forall n\in\mathbb{N}.
\end{equation}
\end{lemma}
\begin{proof}
By the definition of induced matrix norm,
\begin{align}
    &\norm{  {\bm{V}_n}^{-1/2}\bm{X}^{\top}}_{2}=\max_{\norm{v}_{2}=1}\sqrt{v^{\top}\bm{X}{ {\bm{V}_n}^{-1}\bm{X}^{\top}v}}\label{equation:norm of V -1/2 X transpose 1}\\
    &\hspace{24pt}=\lambda_{\max}\Big(\bm{X}  {\bm{V}_n}^{-1}\bm{X}^{T}\Big)\label{equation:norm of V -1/2 X transpose 2}\\
    &\hspace{24pt}=\lambda_{\max}\Big(\bm{X}\big(\bm{X}^{T}\bm{X}+\lambda\bm{I}_{d}\big)^{-1}\bm{X}^{T}\Big)\label{equation:norm of V -1/2 X transpose 3}\\
    &\hspace{24pt}{\leq} \dfrac{\lambda_{\max}(\bm{X}^{\top}\bm{X})}{\lambda_{\max}(\bm{X}^{\top}\bm{X})+\lambda}\leq 1,\label{equation:norm of V -1/2 X transpose 4}
\end{align}
where (\ref{equation:norm of V -1/2 X transpose 4}) follows from the singular value decomposition and $\lambda_{\max}(\bm{X}^{\top}\bm{X})\geq 0$.
\end{proof}

To simplify notation, we use $\bm{X}$ and $\bm{V}$ as a shorthand for $\bm{X}_{n}$ and $\bm{V}_{n}$, respectively.
For convenience, we rewrite $\bm{V}^{-1/2}X^{\top}=[v_1 \cdots v_n]$ as the matrix of $n$ column vectors $\{v_{i}\}_{i=1}^{n}$ (each $v_i\in \mathbb{R}^{d}$) \pch{and show the following property}.
\begin{lemma}
\label{lemma:summation of vi norm less than d}
\pch{Let $v_i\in \mathbb{R}^{d}$ be the $i$-th column of the matrix $\bm{V}^{-1/2}X^{\top}$, for all $1\leq i\leq n$. Then, we have}
\begin{align}
    \sum_{i=1}^{n}\norm{v_{i}}_{2}^{2}\leq d.
\end{align}
\end{lemma}
\begin{proof}[Proof of Lemma \ref{lemma:summation of vi norm less than d}]
Recall that $\lambda_{\max}(\cdot)$ denotes the largest eigenvalue of a square matrix.
We know
\begin{align}
    \sum_{i=1}^{n}\norm{v_{i}}_{2}^{2}&=\text{tr}\Big(\big(\bm{X}  {\bm{V}}^{-1/2}\big)\big({\bm{V}}^{-1/2}\bm{X}^{\top}\big)\Big)\label{equation:summation of vi norm less than d 1}\\
    &=\text{tr}\Big(\big({\bm{V}}^{-1/2}\bm{X}\big)\big(  \bm{X}^{\top}{\bm{V}}^{-1/2}\big)\Big)\label{equation:summation of vi norm less than d 2}\\
    &\leq d\cdot \lambda_{\max}\Big(\big({\bm{V}}^{-1/2}\bm{X}\big)\big(  \bm{X}^{\top}{\bm{V}}^{-1/2}\big)\Big),\label{equation:summation of vi norm less than d 3}
\end{align}
where (\ref{equation:summation of vi norm less than d 2}) follows from the trace of a product being commutative, and (\ref{equation:summation of vi norm less than d 3}) follows since the trace is the sum of all eigenvalues.
Moreover, we have
\begin{align}
    &\lambda_{\max}\big(\big(\bm{X}  {\bm{V}}^{1/2}\big)\big(\bm{X}^{\top}  {\bm{V}}^{-1/2}\big)\big)\label{equation:summation of vi norm less than d 4}\\
    &\hspace{24pt}=\norm{\big(\bm{X}  {\bm{V}}^{1/2}\big)\big(\bm{X}^{\top}  {\bm{V}}^{-1/2}\big)}_{2}\label{equation:summation of vi norm less than d 5}\\
    &\hspace{24pt}\leq \norm{\big(\bm{X}  {\bm{V}}^{1/2}\big)}_{2}\norm{\big(\bm{X}^{\top}  {\bm{V}}^{-1/2}\big)}_{2}\leq 1,\label{equation:summation of vi norm less than d 6}
\end{align}
where (\ref{equation:summation of vi norm less than d 6}) follows from the fact that the $\ell_{2}$-norm is sub-multiplicative.
Therefore, by (\ref{equation:summation of vi norm less than d 1})-(\ref{equation:summation of vi norm less than d 6}), we conclude that
$\sum_{i=1}^{n}\norm{v_{i}}_{2}^{2}\leq d$.
\end{proof}

We are now ready to prove Lemma \ref{lemma:upper bound for first error term}.
\begin{proof}[Proof of Lemma \ref{lemma:upper bound for first error term}]
To simplify notation, we use $\bm{X}$ and $\bm{V}$ as a shorthand for $\bm{X}_{n}$ and $\bm{V}_{n}$, respectively.
To begin with, \pch{we know $f^{-1}(\varepsilon\circ\varepsilon)-\bm{X}\phi_{*}=\frac{1}{M_f}((\varepsilon\circ\varepsilon)-f(\bm{X}\phi_{*}))$}.
Therefore, we have
\begin{align}
    &\norm{\bm{X}(\pch{f^{-1}(\varepsilon\circ\varepsilon)}-\bm{X}\phi_{*})}_{\inv{  {\bm{V}}}}\\
    &\pch{=\frac{1}{M_f}\sqrt{\big(\pch{\varepsilon\circ\varepsilon}- f(\bm{X}\phi_{*})\big)^{\top}\bm{X}\inv{  {\bm{V}}}\bm{X}^{\top}\big(\pch{\varepsilon\circ\varepsilon}- f(\bm{X}\phi_{*})\big)}},
\end{align}
where each element in the vector \pch{$(\varepsilon\circ\varepsilon-f(\bm{X}\phi_{*}))$} is a centered $\chi^{2}_{1}$-distribution with a scaling of \pch{$f(\phi_{*}^{\top}x_{i})$}. 
Defining \pch{$\bm{W}=\text{diag}\big(f(x_{1}^{\top}\phi_{*}),...,f(x_{n}^{\top}\phi_{*})\big)$}, we have
\begin{align}
    &\pch{\norm{\bm{X}(f^{-1}(\varepsilon\circ\varepsilon)-\bm{X}\phi_{*})}_{\inv{  {\bm{V}}}}}\\
    &\hspace{10pt}\pch{=\frac{1}{M_f}\Big[\underbrace{\Big(\varepsilon\circ\varepsilon- f(\bm{X}\phi_{*})\Big)^{\top}\inv{\bm{W}}}_{\textrm{mean=0, variance= 2}}\Big(\bm{W}\bm{X}\inv{  {\bm{V}}}\bm{X}^{\top}\bm{W}\Big)}\\
    &\hspace{24pt}\pch{\underbrace{\inv{\bm{W}}\Big(\varepsilon\circ\varepsilon- f(\bm{X}\phi_{*})\Big)}_{\textrm{mean=0, variance=2}}\Big]^{1/2}}.
\end{align}
We use \pch{$\eta=\inv{\bm{W}}\big(\varepsilon\circ\varepsilon-f(\bm{X}\phi_{*})\big)$} as a shorthand and define $\bm{U}=\big({U}_{ij}\big)=\bm{W}\bm{X}\inv{  {\bm{V}}}\bm{X}^{\top}\bm{W}$. 
By Lemma \ref{lemma:Erdos lemma} and the fact that $\varepsilon(x_1),\cdots,\varepsilon(x_n)$ are mutually independent given the contexts $\{x_i\}_{i=1}^{n}$, we have
\begin{align}
    \mathbb{P}\Big\{\abs{\eta^{\top}\bm{U}\eta-2\cdot\textrm{tr}(\bm{U})}&\geq 2s\Big(\sum_{i=1}^{n}\abs{\bm{U}_{ii}}^{2}\Big)^{1/2}\Big\}\label{equation:P of eta-U-eta 1}\\
    &\leq C_1 \text{exp}(-C_2 \sqrt{s}).\label{equation:P of eta-U-eta 2}
\end{align}
Recall that $\bm{V}^{-1/2}X^{\top}=[v_1 \cdots v_n]$.
The trace of $\bm{U}$ can be upper bounded as
\begin{align}
    &\text{tr}(\bm{U})= \text{tr}(\bm{W}\bm{X}\inv{  {\bm{V}}}\bm{X}^{\top}\bm{W})\label{equation: trace of U 1}\\
    &\hspace{24pt}=\text{tr}\Big(  {\bm{V}}^{-1/2}\bm{X}^{\top}\bm{W}\bm{W}\bm{X}  {\bm{V}}^{-1/2}\Big)\label{equation: trace of U 2}\\
    &\hspace{24pt}\pch{=\sum_{i=1}^{n}f(x_{i}^{\top}\phi_{*})^{2}\cdot \norm{v_{i}}_{2}^{2}}\label{equation: trace of U 3}\\
    &\hspace{24pt}\leq \pch{(\sigma_{\max}^{2})^{2}}\sum_{i=1}^{n}\norm{v_{i}}_{2}^{2}\leq \pch{(\sigma_{\max}^{2})^{2}}d,\label{equation: trace of U 4}
\end{align}
where the last inequality in (\ref{equation: trace of U 4}) follows directly from Lemma \ref{lemma:summation of vi norm less than d}.
Also by the commutative property of the trace operation, we have
\begin{align}
    &\sum_{i=1}^{n}\abs{\bm{U}_{ii}}^{2}\stackrel{(a)}{\leq}\Big(\sum_{i=1}^{n}\bm{U}_{ii}\Big)^{2}\stackrel{(b)}{\leq} \big(\pch{(\sigma_{\max}^{2})^{2}}d\big)^{2},\label{equation:sum of square of diagonal elements}
\end{align}
where (a) follows from $\bm{U}$ being positive semi-definite (all diagonal elements are nonnegative), and (b) follows from (\ref{equation: trace of U 4}). 
\pch{Therefore, by (\ref{equation:P of eta-U-eta 1})-(\ref{equation:sum of square of diagonal elements}), we have}
\begin{align}
    &\mathbb{P}\Big\{\eta^{\top}\bm{U}\eta\geq 2s\cdot \pch{(\sigma_{\max}^{2})^{2}}d+2 \pch{(\sigma_{\max}^{2})^{2}}d\Big\}\\
    &\leq C_{1}\cdot \text{exp}(-C_{2}\sqrt{s}).
\end{align}
By choosing $s=\Big({\dfrac{1}{C_2}\ln{\dfrac{C_1}{\delta}}}\Big)^{2}$, we have
\begin{align}
    \mathbb{P}\Big\{\eta^{\top}\bm{U}\eta\geq 2 \pch{(\sigma_{\max}^{2})^{2}}d\Big(\Big({\dfrac{1}{C_2}\ln{\dfrac{C_1}{\delta}}}\Big)^{2}+1\Big)\Big\}\leq \delta.
\end{align}
Therefore, we conclude that with probability at least $1-\delta$, the following inequality holds
\begin{align}
    &\pch{\norm{\bm{X}(f^{-1}(\varepsilon\circ\varepsilon)-\bm{X}\phi_{*})}_{\inv{  {\bm{V}}}}}\\
    &\hspace{36pt}\pch{\leq \frac{1}{M_f}\sqrt{2\pch{(\sigma_{\max}^{2})^{2}}\cdot d\Big(\Big({\dfrac{1}{C_2}\ln{\dfrac{C_1}{\delta}}}\Big)^{2}+1\Big)}}.
\end{align}
\end{proof}

\subsection{Proof of Lemma \ref{lemma:upper bound for second error term}}
\label{section:appendix:upper bound for second error term}
We first introduce a useful lemma.
\begin{lemma}[Theorem 4.1 in \cite{tropp2012user}]
\label{lemma:Tropp 2012}
Consider a finite sequence $\{\bm{A}_{k}\}$ of fixed self-adjoint matrices of dimension $d\times d$, and let $\{\gamma_k\}$ be a finite sequence of independent standard normal variables. 
Let $\sigma^{2}=\norm{\sum_{k}{\bm{A}_{k}^{2}}}_{2}$.
Then, for all $s\geq 0$,
\begin{equation}
    \mathbb{P}\Big\{\lambda_{\max}\Big(\sum_{k}\gamma_{k}\bm{A}_{k}\Big)\geq s\Big\}\leq d\cdot\exp(-\frac{s^{2}}{2\sigma^{2}}),
\end{equation}
where $\lambda_{\max}(\cdot)$ denotes the largest eigenvalue of a square matrix.
\end{lemma}
Now we are ready to prove Lemma \ref{lemma:upper bound for second error term}.
\begin{proof}[Proof of Lemma \ref{lemma:upper bound for second error term}]
To simplify notation, we use $\bm{X}$ and $\bm{V}$ as a shorthand for $\bm{X}_{n}$ and $\bm{V}_{n}$, respectively.
Recall that $\bm{V}^{-1/2}\bm{X}^{\top}=[v_{1},v_{2},...,v_{n}]$ and define $\bm{A}_{i}=v_{i}v_{i}^{\top}$, for all $i=1,...,n$. 
Note that $\bm{A}_{i}$ is symmetric, for all $i$.
Define an $n\times n$ diagonal matrix $\bm{D}=\text{diag}(\varepsilon_{1},\varepsilon_{2},...,\varepsilon_{n})$. Then we have:
\begin{align}
    &\norm{\bm{X}^{\top}\Big(\varepsilon \circ \big( \bm{X}(\theta_{*}-\widehat{\theta})\big)\Big)}_{\inv{  {\bm{V}}}}\label{equation:upper bound for second term in appendix 1}\\
    &\hspace{2pt}=\norm{  {\bm{V}}^{-1/2}\bm{X}^{\top}\Big(\varepsilon \circ \big( \bm{X}(\theta_{*}-\widehat{\theta})\big)\Big)}_{2}\label{equation:upper bound for second term in appendix 2}\\
    &\hspace{2pt}=\norm{  {\bm{V}}^{-1/2}\bm{X}^{\top} \bm{D} \bm{X}(\theta_{*}-\widehat{\theta})}_{2}\label{equation:upper bound for second term in appendix 3}\\
    &\hspace{2pt}=\norm{  {\bm{V}}^{-1/2}\bm{X}^{\top} \bm{D} \bm{X}  {\bm{V}}^{-1/2}  {\bm{V}}^{1/2}(\theta_{*}-\widehat{\theta})}_{2}\label{equation:upper bound for second term in appendix 4}\\
    &\hspace{2pt}\leq \norm{  {\bm{V}}^{-1/2}\bm{X}^{\top} \bm{D} \bm{X}  {\bm{V}}^{-1/2}}_{2}\cdot\norm{  {\bm{V}}^{1/2}(\theta_{*}-\widehat{\theta})}_{2}\label{equation:upper bound for second term in appendix 5}\\
    &\hspace{2pt}=\norm{  {\bm{V}}^{-1/2}\bm{X}^{\top} \bm{D} \bm{X}  {\bm{V}}^{-1/2}}_{2}\cdot\norm{ \theta_{*}-\widehat{\theta}}_{\bm{V}}.\label{equation:upper bound for second term in appendix 6}
\end{align}
Next, the first term in (\ref{equation:upper bound for second term in appendix 6}) can be expanded into
\begin{align}
    &\norm{  {\bm{V}}^{-1/2}\bm{X}^{\top} \bm{D} \bm{X}  {\bm{V}}^{-1/2}}_{2}\\
    &\hspace{0pt}=\norm{\sum_{i=1}^{n}\varepsilon_{i}v_{i}v_{i}^{\top}}_{2}=\norm{\sum_{i=1}^{n}\dfrac{\varepsilon_{i}}{\sqrt{f(x_{i}^{\top}\phi_{*})}}\cdot\Big(\sqrt{f(x_{i}^{\top}\phi_{*})}\bm{A}_{i}\Big)}_{2}.
\end{align}
Note that $\dfrac{\varepsilon_{i}}{\sqrt{f(x_{i}^{\top}\phi_{*})}}$ is a standard normal random variable, for all $i$.
We also define a $d\times d$ matrix $\bm{\Sigma}=\sum_{i=1}^{n}f(x_{i}^{\top}\phi_{*})\bm{A_{i}}^{2}$.
Then, we have
\begin{align}
    \bm{\Sigma}&=\sum_{i=1}^{n}f(x_{i}^{\top}\phi_{*})\Big(v_{i}v_{i}^{\top}\Big)\Big(v_{i}v_{i}^{\top}\Big)\\
    &=\sum_{i=1}^{n}f(x_{i}^{\top}\phi_{*})\norm{v_{i}}^{2}_{2}v_{i}v_{i}^{\top}.
\end{align}
We also know
\begin{align}
    & \norm{\sum_{i=1}^{n}\bm{A}_{i}}_{2}=\norm{\sum_{i=1}^{n}v_{i}v_{i}^{\top}}_{2}\label{equation:norm of sum of Ai 1}\\
    &\hspace{24pt}=\norm{\Big(  {\bm{V}}^{-1/2}\bm{X}^{\top}\Big)\Big(\bm{X}  {\bm{V}}^{-1/2}\Big)}_{2}\label{equation:norm of sum of Ai 2}\\
    &\hspace{24pt}\leq \norm{\Big(  {\bm{V}}^{-1/2}\bm{X}^{\top}\Big)}_{2}\norm{\Big(\bm{X}  {\bm{V}}^{-1/2}\Big)}_{2}\leq 1,\label{equation:norm of sum of Ai 3}
\end{align}
where (\ref{equation:norm of sum of Ai 3}) follows from Lemma \ref{lemma:norm of V -1/2 X transpose}.
Moreover, we know
\begin{align}
    &\norm{\bm{\Sigma}}_{2}=\norm{\sum_{i=1}^{n}f(x_{i}^{\top}\phi_{*})\norm{v_{i}}^{2}_{2}v_{i}v_{i}^{\top}}_{2}\label{equation:Sigma norm bound 1}\\
    &\hspace{24pt} \leq \norm{d\cdot \pch{\sigma_{\max}^{2}}\sum_{i=1}^{n}v_{i}v_{i}^{T}}_{2}\label{equation:Sigma norm bound 2}\\
    &\hspace{24pt}=d\cdot \pch{\sigma_{\max}^{2}} \norm{\sum_{i=1}^{n}\bm{A}_{i}}\leq d\cdot \pch{\sigma_{\max}^{2}},\label{equation:Sigma norm bound 3}
\end{align}
where (\ref{equation:Sigma norm bound 2}) follows from Lemma \ref{lemma:norm of V -1/2 X transpose}-\ref{lemma:summation of vi norm less than d}, \pch{$f(x_{i}^{\top}\phi_{*})\leq \sigma_{\max}^{2}$}, and that $v_{i}v_{i}^{\top}$ is positive semi-definite, and the last inequality follows directly from (\ref{equation:norm of sum of Ai 3}).
By Lemma \ref{lemma:Tropp 2012} and the fact that $\varepsilon(x_1),\cdots,\varepsilon(x_n)$ are mutually independent given the contexts $\{x_i\}_{i=1}^{n}$, we know that 
\begin{align}
    \mathbb{P}\Big\{\lambda_{\max}\Big(\sum_{i=1}^{n}\varepsilon_{i}\bm{A}_{i}\Big)\geq\sqrt{2\norm{\bm{\Sigma}}_{2}s}\Big\}\leq d \cdot e^{-s}.
\end{align}
Therefore, by choosing $s=\ln({d}/{\delta})$ and the fact that $\lambda_{\max}\Big(\sum_{i=1}^{n}\varepsilon_{i}\bm{A}_{i}\Big)=\norm{\sum_{i=1}^{n}\varepsilon_{i}\bm{A}_{i}}_{2}$, we obtain
\begin{align}
    \mathbb{P}\bigg\{\norm{\sum_{i=1}^{n}\varepsilon_{i}\bm{A}_{i}}_{2}\geq \sqrt{2\pch{\sigma_{\max}^{2}}d\ln({\dfrac{d}{\delta}})}\bigg\}\leq \delta.\label{equation:high probabilty bound for second term}
\end{align}
Finally, by applying Lemma \ref{lemma:theta confidence} and (\ref{equation:high probabilty bound for second term}) to (\ref{equation:upper bound for second term in appendix 6}), we conclude that for any $n\in\mathbb{N}$, for any $\delta>0$, with probability at least $1-\delta$, we have
\begin{equation}
    \norm{\bm{X}_{n}^{\top}\big(\varepsilon\circ \bm{X}_{n}(\theta_{*}-\widehat{\theta}_{n})\big)}_{{\inv{ {\bm{V}}_{n}}}}\leq \alpha^{(1)}_{n}(\delta)\cdot\alpha^{(3)}(\delta).
\end{equation}
\end{proof}

\subsection{Proof of Lemma \ref{lemma:upper bound for third error term}}
\label{section:appendix:upper bound for third error term}
We first introduce a useful lemma on the norm of the Hadamard product of two matrices.
\begin{lemma}
\label{lemma:Frobenius norm Hadamard product}
Given any two matrices $\bm{A}$ and $\bm{B}$ of the same dimension, the following holds:
\begin{align}
    \norm{\bm{A}\circ \bm{B}}_{F}\leq \emph{tr}(\bm{A}\bm{B}^{\top})\leq \norm{\bm{A}}_{2}\cdot \norm{\bm{B}}_{2},
\end{align}
where $\norm{\cdot}$ denotes the Frobenius norm.
When $\bm{A}$ and $\bm{B}$ are vectors, the above degenerates to 
\begin{align}
    \norm{\bm{A}\circ \bm{B}}_{2}\leq \norm{\bm{A}}_{2}\cdot \norm{\bm{B}}_{2}.
\end{align}
\end{lemma}
\begin{proof}[Proof of Lemma \ref{lemma:upper bound for third error term}]
To simplify notation, we use $\bm{X}$ and $\bm{V}$ as a shorthand for $\bm{X}_{n}$ and $\bm{V}_{n}$, respectively.
Let $\bm{M}$ be a positive definite matrix.
We have
\begin{align}
    &\norm{\bm{A}v}_{\bm{M}}=\norm{\bm{M}^{1/2}\bm{A}v}_{2}\leq \norm{\bm{M}^{1/2}\bm{A}}_{2}\cdot\norm{v}_{2},
\end{align}
where the last inequality holds since $\ell_{2}$-norm is sub-multiplicative.
Meanwhile, we also observe that
\begin{align}
    &\Big(\theta_{*}-\widehat{\theta}\Big)^{\top}\bm{X}^{\top}\bm{X}\Big(\theta_{*}-\widehat{\theta}\Big)\\
    &\hspace{10pt}=\Big(\theta_{*}-\widehat{\theta}\Big)^{\top}  {\bm{V}}^{1/2}  {\bm{V}}^{-1/2}\bm{X}^{\top}\bm{X}  {\bm{V}}^{-1/2}  {\bm{V}}^{1/2}\Big(\theta_{*}-\widehat{\theta}\Big)\\
    &\hspace{10pt}=\norm{\Big(\theta_{*}-\widehat{\theta}\Big)^{\top}  {\bm{V}}^{1/2}  {\bm{V}}^{-1/2}\bm{X}^{\top}}_{2}^{2}\\
    &\hspace{10pt}\leq \norm{\Big(\theta_{*}-\widehat{\theta}\Big)^{\top}  {\bm{V}}^{1/2}}_{2}^{2}\norm{  {\bm{V}}^{-1/2}\bm{X}^{\top}}_{2}^{2}\\
    &\hspace{10pt}\leq \norm{\theta_{*}-\widehat{\theta}}_{  {\bm{V}}}^{2}.
\end{align}
Therefore, we know
\begin{align}
    &\norm{\bm{X}^{\top}\Big(\bm{X}\big(\theta_{*}-\widehat{\theta}\big)\circ \bm{X}\big(\theta_{*}-\widehat{\theta}\big) \Big)}_{\inv{  {\bm{V}}}}\label{equation:upper bound third term 1}\\
    &\hspace{2pt}\leq \norm{  {\bm{V}}^{-1/2}\bm{X}^{\top}}_{2}\norm{\Big(\bm{X}\big(\theta_{*}-\widehat{\theta}\big)\circ \bm{X}\big(\theta_{*}-\widehat{\theta}\big) \Big)}_{2}\label{equation:upper bound third term 2}\\
    &\hspace{2pt}\leq 1 \cdot \norm{\bm{X}\big(\theta_{*}-\widehat{\theta}\big)}_{2}^{2}\label{equation:upper bound third term 3}\\
    &\hspace{2pt}\leq 1 \cdot \Big(\big(\theta_{*}-\widehat{\theta}\big)^{\top}\bm{X}^{\top}\bm{X}\big(\theta_{*}-\widehat{\theta}\big)\Big)\label{equation:upper bound third term 4}\\
    &\hspace{2pt}\leq \norm{\theta_{*}-\widehat{\theta}}_{  {\bm{V}}}^{2}\leq (\alpha_{n}^{(1)}(\delta))^{2},\label{equation:upper bound third term 5}
\end{align}
where (\ref{equation:upper bound third term 3}) follows from Lemma \ref{lemma:norm of V -1/2 X transpose} and \ref{lemma:Frobenius norm Hadamard product}, and (\ref{equation:upper bound third term 5}) follows from Lemma \ref{lemma:theta confidence}.
The proof is complete.
\end{proof}


\subsection{Proof of Theorem \ref{theorem:h beta linear approximation}}
\label{section:appendix:h beta linear approximation}
Recall that $h_{\beta}(u,v)=\Big({\Phi\big(\frac{\beta-u}{\sqrt{f(v)}}\big)}\Big)^{-1}.$
We first need the following lemma about Lipschitz smoothness of the function $h_{\beta}(u,v)$.
\begin{lemma}
\label{lemma:hx is Lipschitz smooth}
The function $h_{\beta}(u,v)$ defined in (\ref{equation:hx definition}) is (uniformly) Lipschitz smooth on its domain, i.e., there exists a finite $M_{h}>0$ ($M_h$ is independent of $u$, $v$, and $\beta$) such that for any $\beta$ with $\abs{\beta}\leq B$, for any $u_1,u_2\in[-1,1]$ and $v_1,v_2\in[\sigma_{\min}^{2},\sigma_{\max}^{2}]$,
\begin{equation}
    \abs{\nabla h_{\beta}(u_1,v_1)-\nabla h_{\beta}(u_2,v_2)}\leq M_h\norm{\small{\colvec{2}{u_1}{v_1}-\colvec{2}{u_2}{v_2}}}_{2}.\label{equation:hx Lipschitz smooth}
\end{equation}
Moreover, we have
\begin{align}
    &h_{\beta}(u_2,v_2)-h_{\beta}(u_1,v_1)\leq \\
    &{\small{\colvec{2}{u_2-u_1}{v_2-v_1}}}^{\top}\nabla h_{\beta}(u_1,v_1)+\frac{M_h}{2}\norm{\small{\colvec{2}{u_2-u_1}{v_2-v_1}}}_{2}^{2}.\label{equation:hx Lipschitz smooth approximation}
\end{align}
\end{lemma}
\begin{proof}[Proof of Lemma \ref{lemma:hx is Lipschitz smooth}]
First, it is easy to verify that $h_{\beta}(\cdot,\cdot)$ is twice continuously differentiable on its domain $[-1,1]\times [\sigma_{\min}^{2},\sigma_{\max}^{2}]$ and therefore is Lipschitz smooth, for some finite positive constant $M_h$.
To show that there exists an $M_h$ that is independent of $u,v,\beta$, we need to consider the gradient and Hessian of $h_{\beta}(\cdot,\cdot)$.
Since $h_{\beta}(u,v)$ is a composite function that involves $\Phi(\cdot)$ and $f(\cdot)$, it is straightforward to write down the first and second derivatives of $h_{\beta}(u,v)$ with respect to $u$ and $v$, which depend on $\Phi(\cdot)$, $\Phi'(\cdot)$, $\Phi''(\cdot)$, $f(\cdot)$, $f'(\cdot)$, and $f''(\cdot)$.
Given the facts that for all the $u,v$ and $\beta$ in the domain of interest, we have $\Phi(\frac{\beta-u}{v})\in[\Phi(\frac{-B-1}{\sigma_{\min}^{2}}),1]$, $\Phi{'}(\frac{\beta-u}{v})\in(0,\frac{1}{\sqrt{2\pi}})$, $\abs{\Phi{''}(\frac{\beta-u}{v})}\leq \frac{B+1}{\sigma_{\min}\sqrt{2\pi}}$, and that $f(\cdot), f'(\cdot), f''(\cdot)$ are all bounded, it is easy to verify that such an $M_h$ indeed exists by substituting the above conditions into the first and second derivatives of $h_{\beta}(u,v)$ with respect to $u$ and $v$.
Moreover, by Lemma 3.4 in \cite{bubeck2015convex}, we know that (\ref{equation:hx Lipschitz smooth approximation}) indeed holds.
\end{proof}

\begin{proof}[Proof of Theorem \ref{theorem:h beta linear approximation}]
Define 
\begin{align}
    q_u&:=\sup_{u_0\in(-1,1)}\abs{\frac{\partial h_{\beta}}{\partial u}}\biggr\rvert_{u=u_0},\\
    q_v&:=\sup_{v_0\in(\sigma_{\min}^{2},\sigma_{\max}^{2})}\abs{\frac{\partial h_{\beta}}{\partial v}}\biggr\rvert_{v=v_0}.
\end{align}
By the discussion in the proof of Lemma \ref{lemma:hx is Lipschitz smooth},
we know that $q_u$ and $q_v$ are both positive real numbers.
By substituting $u_1=\theta_{1}^{\top}x$, $u_2=\theta_{2}^{\top}x$, $v_1=f(\phi_{1}^{\top}x)$, and $v_2=f(\phi_{2}^{\top}x)$ into (\ref{equation:hx Lipschitz smooth approximation}), we have
\begin{align}
    &h_{\beta}\big(\theta_{2}^{\top}x,\phi_{2}^{\top}x\big)-h_{\beta}\big(\theta_{1}^{\top}x,\phi_{1}^{\top}x\big)\label{equation:h beta linear approximation on theta and phi 1}\\
    &\hspace{6pt}\leq {\small{\colvec{2}{(\theta_2-\theta_1)^{\top}x}{f(\phi_2^{\top}x)-f(\phi_1^{\top}x)}}}^{\top}\nabla h_{\beta}(\theta_1^{\top}x,f(\phi_1^{\top}x))\label{equation:h beta linear approximation on theta and phi 2}\\
    &\hspace{12pt}+\frac{M_h}{2}\norm{\small{\colvec{2}{(\theta_2-\theta_1)^{\top}x}{f(\phi_2^{\top}x)-f(\phi_1^{\top}x)}}}_{2}^{2}\label{equation:h beta linear approximation on theta and phi 3}
\end{align}
\begin{align}
    &\hspace{6pt}\leq \big(q_u \norm{\theta_2-\theta_1}_{\bm{M}}\cdot\norm{x}_{\bm{M}^{-1}}\\
    &\hspace{12pt}+q_{v}M_{f}\norm{\phi_2-\phi_1}_{\bm{M}}\cdot\norm{x}_{\bm{M}^{-1}}\big)\label{equation:h beta linear approximation on theta and phi 4}\\
    &\hspace{12pt}+\frac{M_h}{2}\big(\norm{\theta_{2}-\theta_{1}}_{\bm{M}}^{2}+M_{f}^{2}\norm{\phi_{2}-\phi_{1}}_{\bm{M}}^{2}\big)\cdot\norm{x}_{\bm{M}^{-1}}\label{equation:h beta linear approximation on theta and phi 5}\\
    &\hspace{6pt}\leq (q_u+M_h) \norm{\theta_2-\theta_1}_{\bm{M}}\cdot\norm{x}_{\bm{M}^{-1}}\label{equation:h beta linear approximation on theta and phi 6}\\
    &\hspace{12pt}+M_{f}(q_{v}+M_{h}M_{f}L)\norm{\phi_2-\phi_1}_{\bm{M}}\cdot\norm{x}_{\bm{M}^{-1}},\label{equation:h beta linear approximation on theta and phi 7}
\end{align}
where (\ref{equation:h beta linear approximation on theta and phi 4})-(\ref{equation:h beta linear approximation on theta and phi 5}) follow from the Cauchy-Schwarz inequality and the fact that $f(\cdot)$ is Lipschitz continuous, and (\ref{equation:h beta linear approximation on theta and phi 6})-(\ref{equation:h beta linear approximation on theta and phi 7}) follow from the facts that $\norm{x}_{2}\leq 1$, $\norm{\theta_2-\theta_1}_{2}\leq 2$, and $\norm{\phi_2-\phi_1}_{2}\leq 2L$.
By letting $C_3=q_u+M_h$ and $C_4=M_f(q_v+M_{h}M_{f}L)$, we conclude (\ref{equation:h beta approx theorem 1})-(\ref{equation:h beta approx theorem 2}) indeed holds with $C_3$ and $C_4$ being independent of $\theta_1,\theta_2,\phi_1,\phi_2$, and $\beta$. 
\end{proof}

\subsection{Proof of Lemma \ref{lemma:Q is an upper confidence bound}}
\label{section:appendix:Q is an upper confidence bound}
\begin{proof}
By Theorem \ref{theorem:h beta linear approximation} and (\ref{equation:xi definition}), we know 
\begin{align}
    &Q^{\textrm{HR}}_{t+1}(x)-\pch{h_{\beta_{t+1}}}(\theta_{*}^{\top}x, \phi_{*}^{\top}x)\label{equation:show ucb 1}\\
    &=\pch{h_{\beta_{t+1}}}(\widehat{\theta_t}^{\top}x, \widehat{\phi_t}^{\top}x)+\xi_{t}(\delta)\norm{x}_{ {\bm{V}}_{t}^{-1}}-\pch{h_{\beta_{t+1}}}(\theta_{*}^{\top}x, \phi_{*}^{\top}x)\label{equation:show ucb 2}\\
    &\leq 2\xi_{t}(\delta)\norm{x}_{ {\bm{V}}_{t}^{-1}}.\label{equation:show ucb 3}
\end{align}
Similarly, by switching the roles of $\theta_{*}^{\top},\phi_{*}^{\top}$ and $\widehat{\theta_t}^{\top},\widehat{\phi_t}^{\top}$ in (\ref{equation:show ucb 2}), we have
\begin{align}
    &Q^{\textrm{HR}}_{t+1}(x)-\pch{h_{\beta_{t+1}}}(\theta_{*}^{\top}x, \phi_{*}^{\top}x)\geq 0.
\end{align}
\end{proof}

\subsection{Proof of Theorem \ref{theorem:regret}}
\label{section:appendix:regret}
\begin{proof}
For each user $t$, let $\pi_{t}^{\textrm{HR}}=\{x_{t,1},x_{t,2},\cdots\}$ denote the action sequence under the HR-UCB policy.
Under HR-UCB, $\widehat{\theta}_{t}$ and $\widehat{\phi}_{t}$ are updated only after the departure of each user.
This fact implies that $x_{t,i}=x_{t,j}$, for all $i,j$.
Therefore, we can use $x_t$ to denote the action chosen by HR-UCB for the user $t$, to simplify notation. 
Let $\overline{R}_{t}^{\textrm{HR}}$ denote the expected lifetime of user $t$ under HR-UCB.
Similar to (\ref{equation:oracle per-user reward}), we have
\begin{align}
        \overline{R}_{t}^{\textrm{HR}}&=\Big({\Phi\Big(\frac{\beta_{t}-\theta_{*}^{\top}{x}_{t}}{\sqrt{f(\phi_{*}^{\top}{x}_{t})}}\Big)}\Big)^{-1}
        =h_{\beta_t}(\theta_{*}^{\top}x_t,\phi_{*}^{\top}x_t).\label{equation:HR-UCB per-user reward}
\end{align}
Recall that $\pi^{\textrm{oracle}}$ and $x^{*}_t$ denote the oracle policy and the context of the action of the oracle policy for user $t$, respectively.
We compute the pseudo regret of HR-UCB as
\begin{align}
    \textrm{Regret}_{T}&=\sum_{t=1}^{T}\overline{R}_{t}^{{*}}-\overline{R}_{t}^{\textrm{HR}}\label{equation:regret appendix 1}\\
    &=\sum_{t=1}^{T}h_{\beta_t}\big(\theta_{*}^{\top}x_{t}^{*},\phi_{*}^{\top}x_{t}^{*}\big)-h_{\beta_t}\big(\theta_{*}^{\top}x_{t},\phi_{*}^{\top}x_{t}\big)\label{equation:regret appendix 2}.
\end{align}
To simplify notation, we use $w_t$ as a shorthand for $h_{\beta_t}\big(\theta_{*}^{\top}x_{t}^{*},\phi_{*}^{\top}x_{t}^{*}\big)-h_{\beta_t}\big(\theta_{*}^{\top}x_{t},\phi_{*}^{\top}x_{t}\big)$.
Given any $\delta>0$, define an event $E_{\delta}$ in which (\ref{equation:theta confidence set}) and (\ref{equation:phi confidence set}) hold under the given $\delta$, for all $t\in\mathbb{N}$.
By Lemma \ref{lemma:theta confidence} and Theorem \ref{theorem:phi confidence}, we know that the event $E_{\delta}$ occurs with probability at least \pch{$1-3\delta$}.
Therefore, with probability at least \pch{$1-3\delta$}, for all $t\in\mathbb{N}$,
\begin{align}
    w_t&\leq Q^{\textrm{HR}}_{t}(x_{t}^{*})-h_{\beta_t}\big(\theta_{*}^{\top}x_{t},\phi_{*}^{\top}x_{t}\big)\label{equation:immediate regret 2}\\
    &\leq Q^{\textrm{HR}}_{t}(x_{t})-h_{\beta_t}\big(\theta_{*}^{\top}x_{t},\phi_{*}^{\top}x_{t}\big)\label{equation:immediate regret 3}\\
    &=h_{\beta_t}\big(\theta_{*}^{\top}x_{t},\phi_{*}^{\top}x_{t}\big)+\pch{\xi_{t-1}(\delta)}\norm{x_t}_{  \pch{{\bm{V}}_{t-1}^{-1}}}\label{equation:immediate regret 4}\\
    &\hspace{24pt}-h_{\beta_t}\big(\theta_{*}^{\top}x_{t},\phi_{*}^{\top}x_{t}\big)\label{equation:immediate regret 5}\\
    &\leq 2\pch{\xi_{t-1}(\delta)}\cdot\norm{x_t}_{\pch{  {{{\bm{V}}}}_{t-1}^{-1}}},\label{equation:immediate regret 6}
\end{align}
where (\ref{equation:immediate regret 2}) and (\ref{equation:immediate regret 4}) follow directly from the definition of the UCB index, (\ref{equation:immediate regret 3}) follows from the design of HR-UCB algorithm, and (\ref{equation:immediate regret 6}) is a direct result under the event $E_\delta$.
Now, we are ready to conclude that with probability at least \pch{$1-3\delta$}, we have
\begin{align}
    \text{Regret}_{T}&=\sum_{t=1}^{T}w_t\leq \sqrt{T\sum_{t=1}^{T}w_{t}^{2}}\label{equation:regret final 1}\\
    &\leq\sqrt{4\xi_{T}^{2}(\delta)T\sum_{t=1}^{T}\min\{\norm{x_t}^{2}_{\pch{  {\bm{V}}_{t-1}^{-1}}},1\}}\label{equation:regret final 2}\\
    &\leq \sqrt{8\xi_{T}^{2}(\delta)T\cdot d\log\Big(\pch{\frac{\mathcal{S}(T)+\lambda d}{\lambda d}}\Big)} \label{equation:regret final 4},
\end{align}
where (\ref{equation:regret final 1}) follows from the Cauchy-Schwarz inequality, (\ref{equation:regret final 2}) follows from the fact that $\xi_{t}(\delta)$ is an increasing function in $t$,
and (\ref{equation:regret final 4}) follows from Lemma 10 and 11 in \cite{abbasi2011improved} \pch{and the fact that ${\bm{V}}_{t}=\lambda \bm{I}_{d}+\bm{X}_{t}^{\top}\bm{X}_{t}=\lambda \bm{I}_{d}+\sum_{i=1}^{t}x_i x_{i}^{\top}$}.
By substituting \pch{$\xi_{T}(\delta)$} into (\ref{equation:regret final 4}) and using the fact that $\mathcal{S}(T)\leq \Gamma(T)$, we know
\begin{equation}
    \text{Regret}_{T}=\mathcal{O}\bigg(\sqrt{T\pch{\log \Gamma(T)}\cdot\Big(\log\big(\Gamma(T)\big)+\log(\frac{1}{\delta})\Big)^{2}}\bigg).
\end{equation}
By choosing $\Gamma(T)=KT$ for some constant $K>0$, we thereby conclude that
\begin{equation}
    \text{Regret}_{T}=\mathcal{O}\bigg(\sqrt{T\log T\cdot\Big(\log T+\log(\frac{1}{\delta})\Big)^{2}}\bigg).
\end{equation}
The proof is complete.
\end{proof}


\end{document}